\documentclass{article}

\PassOptionsToPackage{numbers, compress}{natbib}

\usepackage{algcompatible}

\usepackage[final]{neurips_2022}

\usepackage[utf8]{inputenc} %
\usepackage[T1]{fontenc}    %
\usepackage{amsmath}
\usepackage{amsfonts}       %
\usepackage{url}            %
\usepackage{booktabs}       %
\usepackage{nicefrac}       %
\usepackage{microtype}      %
\usepackage{xcolor}         %
\usepackage{bm}
\usepackage{diagbox}
\usepackage{oplotsymbl}
\usepackage{pgfplots}
\usepackage{wrapfig}
\usetikzlibrary{patterns}

\usepackage[ruled,vlined]{algorithm2e}

\SetKwInput{KwInput}{Input}                %
\SetKwInput{KwOutput}{Output}              %

\newif\ifdraft
\draftfalse

\usepackage{xparse}
\usepackage{xspace}
\usepackage{algorithm}%
\usepackage{fixltx2e}
\usepackage{tikz}
\usepackage{float}
\usepackage{multirow}
\usepackage{multicol}
\usepackage{colortbl}
\usepackage{array}
\newcommand{\PreserveBackslash}[1]{\let\temp=\\#1\let\\=\temp}
\newcolumntype{C}[1]{>{\PreserveBackslash\centering}p{#1}}
\newcolumntype{R}[1]{>{\PreserveBackslash\raggedleft}p{#1}}
\newcolumntype{L}[1]{>{\PreserveBackslash\raggedright}p{#1}}

\usepackage{microtype}
\usepackage{graphicx}
\usepackage{mathtools}
\usepackage{float}
\usepackage{subcaption}
\usepackage{booktabs} %
\usepackage[shortlabels]{enumitem}

\usepackage{amssymb}%
\usepackage{pifont}%

\usepackage{bbold}

\usepackage{hyperref,amssymb,enumitem}

\setlist[itemize]{leftmargin=*}
\setlist[enumerate]{leftmargin=*}

\newcommand*{\rej}{{\ooalign{\lower.3ex\hbox{$\sqcup$}\cr\raise.4ex\hbox{$\sqcap$}}}}

\newenvironment{proof}{\emph{Proof:}}{\hfill$\square$}

\usepackage{stackengine}

\usepackage{xcolor}

\renewcommand{\algorithmicrequire}{\textbf{Input: }}
\renewcommand{\algorithmicensure}{\textbf{Output: }}

\newcommand{\ie}{\textit{i.e.,}\@\xspace}
\newcommand{\eg}{\textit{e.g.,}\@\xspace}

\DeclareRobustCommand\encircle[1]{\tikz[baseline=(char.base)]{\node[shape=circle,fill,inner sep=1pt] (char) {\textcolor{white}{#1}}}}

\newtheorem{theorem}{Theorem}

\graphicspath{{images/}}

\usepackage{booktabs,arydshln}
\makeatletter
\def\adl@drawiv#1#2#3{%
        \hskip.5\tabcolsep
        \xleaders#3{#2.5\@tempdimb #1{1}#2.5\@tempdimb}%
                #2\z@ plus1fil minus1fil\relax
        \hskip.5\tabcolsep}
\newcommand{\cdashlinelr}[1]{%
  \noalign{\vskip\aboverulesep
           \global\let\@dashdrawstore\adl@draw
           \global\let\adl@draw\adl@drawiv}
  \cdashline{#1}
  \noalign{\global\let\adl@draw\@dashdrawstore
           \vskip\belowrulesep}}
\makeatother

\usepackage{cleveref}
\crefalias{prop}{proposition} %

\newcommand{\nlp}[1]{}
\newcommand{\prob}{log-likelihood\xspace}
\newcommand{\probs}{log-likelihoods\xspace}

\newcolumntype{x}[1]{>{\centering\arraybackslash\hspace{0pt}}p{#1}}
\newcommand{\new}[1]{\textcolor{black}{#1}}

\ifdraft
\usepackage[textsize=tiny]{todonotes}
\newcommand{\mynote}[1]{\textcolor{red}{[note: #1]}}

\newcommand{\mytodo}[1]{\textcolor{red}{[todo: #1]}}
\newcommand{\mycomment}[1]{\textcolor{red}{[comment: #1]}}
\newcommand{\chris}[1]{\textcolor{red}{Chris: #1}}
\newcommand{\ahmad}[1]{\textcolor{darkpastelgreen}{[Ahmad: #1]}}
\newcommand{\vinith}[1]{\textcolor{blue}{Vinith: #1}}
\newcommand{\yunxiang}[1]{\textcolor{cyan}{Yunxiang: #1}}
\newcommand{\xiao}[1]{\textcolor{blue}{xiao: #1}}
\definecolor{chocolate(traditional)}{rgb}{0.48, 0.25, 0.0}

\definecolor{darkpastelgreen}{rgb}{0.01, 0.75, 0.24}
\newcommand{\natalie}[1]{\textcolor{darkpastelgreen}{natalie: #1}}
\definecolor{pistachio}{rgb}{0.58, 0.77, 0.45}

\newcommand{\jonas}[1]{\textcolor{violet}{[Jonas: #1]}}

\newcommand{\adelin}[1]{\textcolor{red}{[Adelin: #1]}}
\newcommand{\mohammad}[1]{\textcolor{red}{[Mohammad: #1]}}

\definecolor{amber(sae/ece)}{rgb}{1.0, 0.49, 0.0}
\newcommand\adam[1]{{\textcolor{red}{[Adam: #1]}}}
\newcommand{\sierra}[1]{\textcolor{blue}{[Sierra: #1]}}
\newcommand{\armin}[1]{\textcolor{cyan}{[Armin: #1]}}
\newcommand{\nikita}[1]{\textcolor{cyan}{[Nikita: #1]}}
\newcommand{\franzi}[1]{\textcolor{purple}{[Franzi: #1]}}
\newcommand{\yannis}[1]{\textcolor{blue}{[Yannis: #1]}}

\else

\usepackage[disable]{todonotes}
\newcommand{\chris}[1]{}
\newcommand{\franzi}[1]{}
\newcommand{\vinith}[1]{}
\newcommand{\adam}[1]{}
\newcommand{\yunxiang}[1]{}
\newcommand{\natalie}[1]{}
\newcommand{\jonas}[1]{}
\newcommand{\adelin}[1]{}
\newcommand{\mynote}[1]{}
\newcommand{\xiao}[1]{}
\newcommand{\mytodo}[1]{}
\newcommand{\mycomment}[1]{}
\newcommand{\ahmad}[1]{}
\newcommand{\mohammad}[1]{}
\newcommand{\sierra}[1]{}
\newcommand{\armin}[1]{}
\newcommand{\nikita}[1]{}

\fi

\usepackage{amsmath,amsfonts,bm}

\def\eqref#1{equation~\ref{#1}}

\def\1{\bm{1}}

\DeclareMathAlphabet{\mathsfit}{\encodingdefault}{\sfdefault}{m}{sl}
\SetMathAlphabet{\mathsfit}{bold}{\encodingdefault}{\sfdefault}{bx}{n}

\def\gD{{\mathcal{D}}}
\def\gE{{\mathcal{E}}}

\def\gI{{\mathcal{I}}}

\def\gR{{\mathcal{R}}}
\def\gS{{\mathcal{S}}}

\def\gV{{\mathcal{V}}}

\usepackage{hyperref}       %
\newcommand{\ourtitle}{Dataset Inference for Self-Supervised Models}
\title{\ourtitle}

\author{%
  Adam Dziedzic\thanks{Corresponding and leading author: adam.dziedzic@utoronto.ca}\ \ \thanks{Equal contribution.}\ , Haonan Duan$^{\dagger}$, Muhammad Ahmad Kaleem$^{\dagger}$, Nikita Dhawan,\\ 
  \textbf{Jonas Guan, Yannis Cattan, Franziska Boenisch, Nicolas Papernot} \\
  University of Toronto and Vector Institute \\ 
}

\begin{document}

\maketitle
\vspace{-1em}
\begin{abstract}
\vspace{-1em}
Self-supervised models are increasingly prevalent in machine learning (ML) since they reduce the need for expensively labeled data. Because of their versatility in downstream applications, they are increasingly used as a service exposed via public APIs. At the same time, these encoder models are particularly vulnerable to model stealing attacks due to the high dimensionality of vector representations they output. Yet, encoders remain undefended: existing mitigation strategies for stealing attacks focus on supervised learning. We introduce a new dataset inference defense, which uses the private training set of the victim encoder model to attribute its ownership in the event of stealing. The intuition is that the log-likelihood of an encoder's output representations is higher on the victim's training data than on test data if it is stolen from the victim, but not if it is independently trained. We compute this log-likelihood using density estimation models. As part of our evaluation, we also propose measuring the fidelity of stolen encoders and quantifying the effectiveness of the theft detection without involving downstream tasks; instead, we leverage mutual information and distance measurements. Our extensive empirical results in the vision domain demonstrate that dataset inference is a promising direction for defending self-supervised models against model stealing.

\end{abstract}

\section{Introduction}

The self-supervised learning (SSL) paradigm enables pre-training models with unlabeled data to learn generally useful domain knowledge and then transfer the knowledge to solve specific downstream tasks.
The ability to learn from unlabeled data alleviates the high costs of labeling large datasets~\citep{kapoor2007selective}, and the transfer learning setup reduces the computational costs of retraining.
These advantages have made SSL increasingly popular~\citep{ssl-survey2021} in domains like vision~\citep{simclr}, language~\citep{devlin2018bert}, and bioinformatics~\citep{bertbio2019}.

Recently, commercial service providers like Cohere~\citep{Cohere} and OpenAI~\citep{OpenAI} %
began offering paid query access to trained SSL encoders over public APIs.
This exposes the encoders to black-box extraction attacks, \ie model stealing.
In a model stealing attack, an attacker aims to train an approximate copy 
of a victim model by submitting carefully chosen queries and observing the victim's outputs.
The high costs of data collection, preprocessing and model training make encoders valuable targets for stealing.
For example, the training data of CLIP includes 400 million image and text pairs~\citep{clip}, while computation costs of training a large language model can exceed one million USD~\citep{bertTrainCost}.
The threat of model stealing in SSL is real: researchers have demonstrated that encoders can be stolen at a fraction of the victim's training cost~\citep{SSLextraction,ContSteal}.
Yet, most current defenses are designed for supervised models~\citep{powDefense,juuti2019prada,orekondy2019prediction} and cannot be directly applied to encoders~\citep{SSLextraction}.

\begin{figure}[t]
\vskip 0in
\begin{center}
\centerline{\includegraphics[width=0.7\columnwidth]{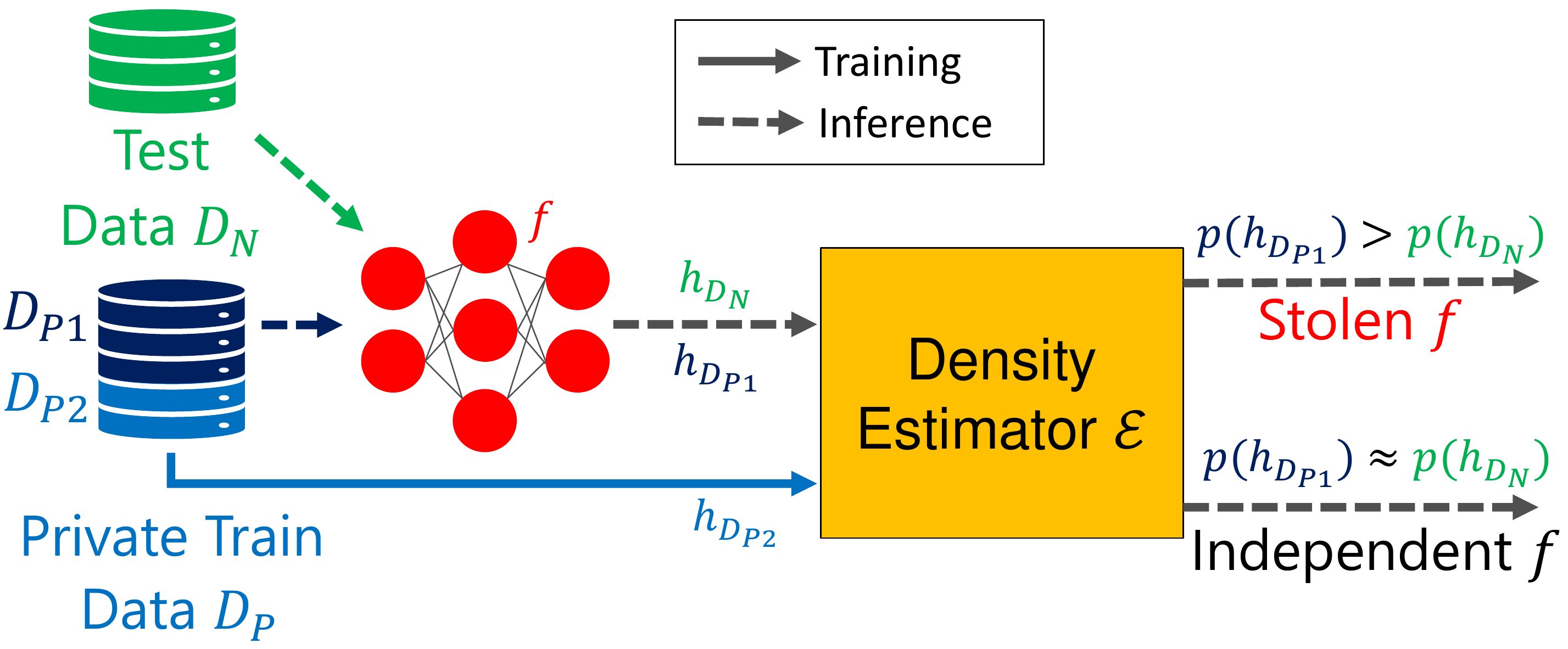}}
\caption{
\label{fig:dataset-inference-only-test}
\textbf{Ownership Resolution for Encoders.} 
\small{
\new{First, an arbitrator trains density estimator $\gE$: divide $D_P$ into non-overlapping partitions $D_{P1}$ and $D_{P2}$, and train density estimator $\gE$ using the representations $h_{D_{P2}}$ of $f$ on $D_{P2}$. Next, the arbitrator performs dataset inference: apply $\gE$ on the representations of $D_{P1}$ and $D_{N}$ of the encoder $f$. For a stolen encoder, the \prob of the representations $h_{D_{P1}}$ is significantly higher than $h_{D_N}$, while, for an independent encoder, the \probs of the representations are not significantly different.}
}
}
\end{center}
\vskip -0.5in
\end{figure}

Dataset inference~\citep{maini2021dataset} is a state-of-the-art defense against model stealing in the supervised learning setting. The defense provides ownership resolution: it enables the model owner to make a strong statistical claim that a given model is a stolen copy of their own model by showing that this model is derivative of their own private training data.
Dataset inference does not require retraining or overfitting the model to any form of explicit watermark~\cite{watermarks2021} and has been shown to resist attacks from adaptive adversaries~\citep{maini2021dataset}.
These properties make the defense particularly attractive for SSL, as large encoders can be expensive to retrain, and performance is paramount because it carries over to all downstream tasks.
However, the original dataset inference algorithm from~\cite{maini2021dataset} cannot be applied to encoders, because it relies on computing distances between data points and decision boundaries.
These decision boundaries do not exist in SSL encoders since they are trained on unlabeled data.

We introduce a new dataset inference method (\Cref{fig:dataset-inference-only-test}) to defend against model stealing for encoders. %
Our algorithm is suitable for the high-dimensional outputs of SSL encoders and does not rely on labeled data or decision boundaries.
Instead, it relies solely on the private training data of the victim encoder as a signature. Moreover, our algorithm retains the advantages of dataset inference for supervised models~\cite{maini2021dataset}, namely, it does not require retraining or overfitting the SSL encoder.

Our key intuition is to identify stolen encoders by characterizing differences between an encoder's representations on its training data vs on unseen test data.
The victim encoder and its derivatives, such as stolen copies, exhibit different behavior on the victim's private training data than on test data; independently trained encoders do not.
These differences exist because encoders overfit to training data~\citep{membershipInferenceSSL2021,liu2021encodermi}.
Although for well-trained encoders the effect is minimal on any given data point, we show that when aggregated over many training points it provides a statistically strong signal.
To identify the differences, we train a Gaussian Mixture Model (GMM), as an efficient general approximation, %
to model the distribution of an encoder's data representations from its training domain.
We then use the GMMs to predict the \prob of the encoders' representations of the victim's training set and a test set; derivatives of the victim encoder will have a higher
\prob on the training set than on the test set.
We perform experiments on five datasets from the vision domain and show that we are able to distinguish between stolen and independent encoders even in cases when adversaries obfuscate the representations from the stolen encoders to hide the theft (\eg by shuffling the elements in the representation vectors or applying to them some form of a linear transformation).%

As part of our evaluation, we also introduce new metrics to measure the fidelity of stolen encoders without involving downstream tasks and to quantify the effectiveness of theft detection. We compute scores directly on the representations using tools from information theory and distance metrics. These methods work well because losses used for stealing encoders directly minimize distances between representations of victim and stolen encoders.
Our mutual information score to assess the quality of the stolen encoders is robust against obfuscations that an adversary might apply to the representations returned by a stolen encoder. Without any obfuscation, our cosine similarity score shows a clearer distinction between stolen and independent encoders. Finally, using these metrics, we observe that the higher the quality of the stolen encoders, the more confident our dataset inference defense becomes.

Our main contributions are as follows:
\begin{itemize}
    \item We propose a new defense against model stealing attacks on encoders, by combining dataset inference with density estimation models for ownership resolution on unlabeled data. 
    \item We are the first to design new metrics that quantify the quality of stolen encoders, which are derived from the mutual information and distances between representations.
    \item We evaluate our defense using five datasets from the computer vision domain\nlp{~and NLP,} and show that our defense can successfully identify stolen encoders with a strong statistical significance.%
\end{itemize}
\section{Related Work}
\label{sec:related-work}

In model stealing, an adversary queries the victim model, obtains outputs, and uses them to recreate a copy of the victim~\citep{pred_apis}.
This is most commonly performed with black-box access, \eg via a public API.
When stealing encoders in SSL, the goal of an adversary is to extract high-quality embeddings either to train a stolen copy that achieves high performance on downstream tasks, or to obtain faithful replicas of the victim's embeddings on the same inputs.
Stolen encoders might be further used for model reselling, backdoor attacks, or membership inference~\citep{SSLguard}.

While most past research on model stealing and defenses focuses on classifiers trained via supervised learning, recent work constructed new attacks that target encoders~\citep{SSLextraction, ContSteal}.
The main differences between the attacks in these settings are that the outputs of encoders leak more information due to their higher dimensionalities~\citep{ContSteal}, and the attacks require different loss functions. %
Inspired by contrastive learning, \textit{Cont-Steal}~\citep{ContSteal} provides a method of stealing encoders using a loss function based on InfoNCE~\citep{simclr}.
SSL extraction~\citep{SSLextraction}---a general Siamese-network-based framework for stealing encoders---leverages losses including mean squared error, InfoNCE, Soft Nearest Neighbor, and Wasserstein distance.
The authors empirically show that an adversary can steal an ImageNet victim encoder in less than a fifth of the queries required for training.%

\new{Proof of Learning (PoL)~\cite{jia2021proof} is a reactive defense that involves the defender claiming ownership of a model by showing incremental updates of the model training. It is a complementary method to dataset inference, which instead identifies a stolen model. PoL could be applied directly to SSL encoders, however, it requires an expensive verification process, where the verifier needs to perform model updates and the prover needs to save intermediate weights of the model, which is more expensive than dataset inference with GMMs.}
Unfortunately, other current defenses against model stealing for supervised learning are inadequate for defending encoders, and adjusting them to the specificities of encoders is non-trivial~\citep{SSLextraction}.
One line of approach is watermarking~\citep{SSLguard, SSLextraction}, where the defender embeds a secret trigger into the victim encoder during training to determine ownership at test time. %
However, watermarking-based defenses have two significant disadvantages.
First, researchers have repeatedly shown that adaptive attackers can remove watermarks without severely affecting model performance~\citep{chen2021refit, entangled_wm, shafieinejad2021robustness, wang2019neural}, \eg through pruning, fine-tuning, rounding or performing backdoor removal~\citep{watermarks2021}.
Second, the watermark must be embedded during training; if a model is already trained, or if a watermark defense needs to be updated, the model must be retrained.
This is not practical for large encoders.

Another state-of-the-art defense against model stealing in the supervised setting is dataset inference~\citep{maini2021dataset} which addresses these disadvantages.
However, the adaptation of dataset inference to encoders is difficult, because (1) the algorithm~\citep{maini2021dataset} relies on decision boundaries, which do not exist for encoders; and (2) encoders are less prone to overfitting, which provides the signal for dataset inference~\citep{SSLextraction, maini2021dataset}.
Therefore, naive approaches like computing the loss of representations or distances between train and test sets are ineffective~\citep{SSLextraction}.
To overcome these issues, we extract more signals from the representations by estimating their densities for train and test sets, as described in Section~\ref{sec:method}.

When it comes to comparing signals within representations, prior work has considered measuring the similarity between different representations. This has led to the proposal of various similarity metrics   
including canonical correlation analysis (CCA)~\citep{morcos2018insights}, centered kernel alignment (CKA)~\citep{similaritynn}, and the orthogonal Procrustes distance~\citep{ding2021grounding}, which use methods from linear regression, principal component analysis (PCA), and singular value decomposition (SVD). However, these metrics are very general and complex. They have also been shown to disagree in some cases~\citep{ding2021grounding}. Since we can only access the final embeddings from encoders, we design metrics more closely related to our setting.

\nlp{Thus far, methods for stealing encoders mainly focused on the computer vision domain.
Previous work in the NLP domain performs stealing against models that are fine-tuned based on a given pre-trained language encoder~\citep{krishna2019thieves, zanella2021grey}.
The setup of these works differs from ours which is concerned with the stealing of the pre-trained encoders themselves.
There exists work on distillation~\cite{hinton2015distilling} (i.e., transferring knowledge from a large teacher model to a smaller student model) of NLP models~\cite{jiao2019tinybert}.
For example, TinyBERT~\cite{jiao2019tinybert} can be distilled from a regular BERT model~\cite{devlin2018bert}. While methods of NLP model distillation can also be applied to steal such models, the aims of distillation and model stealing are different:
while distillation focuses on the compression aspect, an adversary performing model stealing is primarily interested in the accuracy of the extracted model.
Furthermore, distillation usually requires access to the original training dataset (or data from the same distribution), whereas, in model stealing settings, the training data is often unknown to an adversary.
\todo{Jonas: I think this next paragraph can be cut if we need space; our focus should be on dataset inference rather than domains of model stealing. We use the NLP domain solely to strengthen the claims we make with our defense}
}

\section{Defense Method}
\label{sec:method}

Dataset inference serves as a defense against model stealing.
It enables the model owner or a third-party arbitrator to attribute the ownership of a model in the event that it is stolen. %
The idea is to take advantage of the effects of knowledge from the victim's training set %
and to use that as a signature for attributing ownership.
Given a well-trained encoder, the effects are small on any single data point; however, when aggregated over many points in the training set, they collectively provide a strong statistical signal for dataset inference. 
As depicted in Figure~\ref{fig:di-intuition}, for a victim that leverages the private data $D_P$ during training or a stolen copy, we can identify a difference between 
\newcommand\xintu{0.35}
\newcommand\vskipintu{-0.30in}
\begin{wrapfigure}{r}{\xintu\textwidth}
\vskip \vskipintu
\begin{center}
\centerline{\includegraphics[width=\xintu\columnwidth,trim={0 0 0 0}]{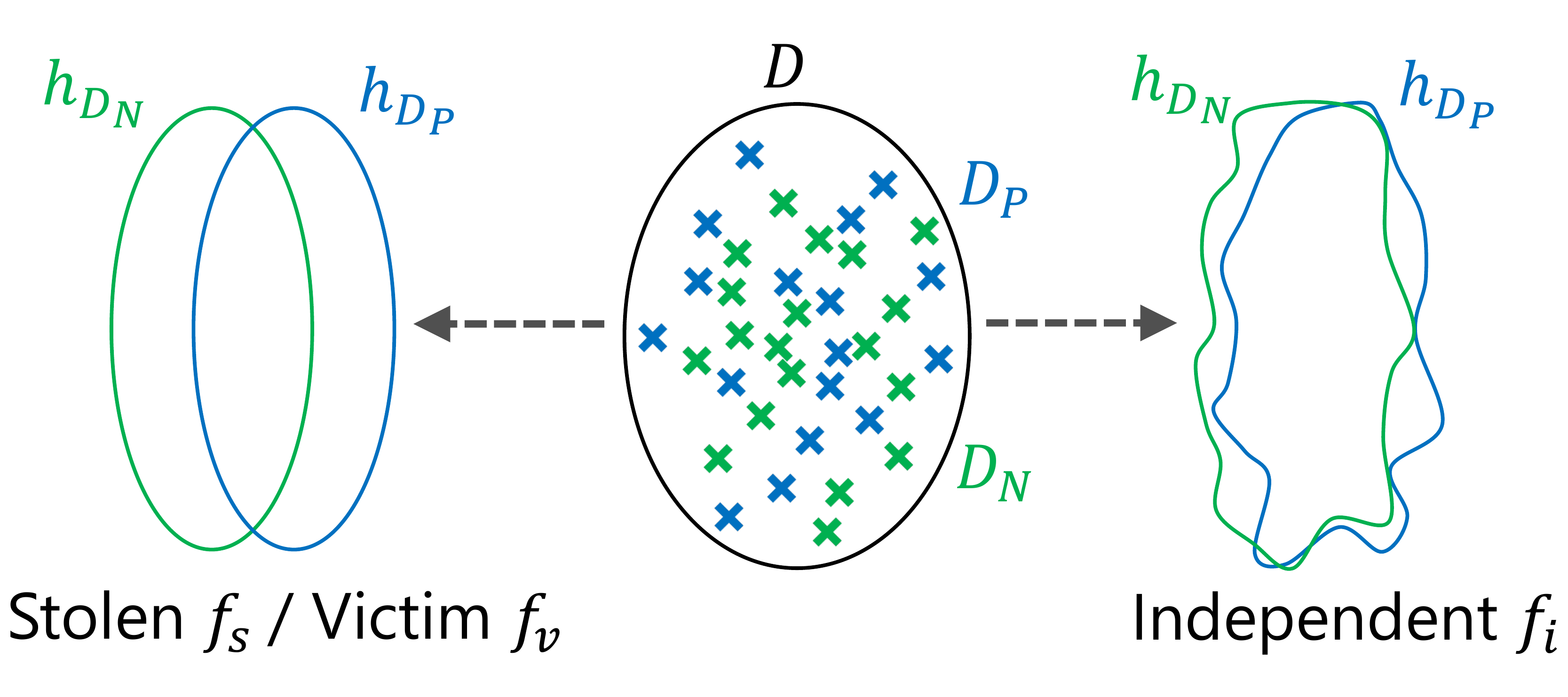}}
\caption{
\label{fig:di-intuition}
\textbf{Dataset Inference - Intuition.} $D_P$ and $D_N$ come from the same distribution. For independent encoders, their representations $h_{D_N}$ and $h_{D_P}$ are i.i.d. while for victim/stolen encoders, $h_{D_N}$ and $h_{D_P}$ induce different distributions.
}
\end{center}
\vskip \vskipintu
\end{wrapfigure}
distributions of the train data's representations~$h_{D_P}$ and the test data's representations~$h_{D_N}$ while the distributions of these representations from an independent encoder cannot be distinguished.
We use this signal to determine whether a model is a derivative of the victim's training data, \ie either directly trained on the data or stolen from the victim.
To capture the signal, we first partition the victim encoder's training data into two subsets and use one to train Gaussian Mixture Models (GMMs) with the aim of modeling the distribution of representations of data from the encoder's domain.
Then, we apply the GMMs to perform dataset inference by measuring the \prob of the encoder's representations of the remaining training data vs some test data (see Figure~\ref{fig:dataset-inference-only-test}).
For the victim encoder or its stolen copies, the \prob of training data representations is significantly higher than the one of test data.
We use this to construct statistical t-tests to determine whether a model is stolen, and present our empirical results in Section~\ref{sec:eval}.

\subsection{Threat Model}
As described in Figure \ref{fig:dataset-inference-overview}, we consider a victim encoder $f_v$ 
trained on a private training dataset $D_P$. An adversary with black-box access to $f_v$ trains a stolen encoder $f_s$ by querying the victim with data points from a dataset $D_S$ to obtain representations. These representations are then used as part of the training objective for the stolen encoder.
During the dataset inference, we assume the presence of a third-party arbitrator, such as law enforcement, with white-box access to the victim's training data, as well as all encoders. 
Additionally, the arbitrator requires the test dataset $D_N$ from the same distribution as $D_P$ to perform dataset inference.
An independent encoder $f_i$ is trained with no access to the victim's private training set $D_P$ and without any queries to the victim's encoder.
It is used as a baseline for the ownership resolution in the dataset inference.

\subsection{Density Estimation of Representations}
\label{sec:density-estimation}
Representations from encoders contain rich features for given inputs. We analyze the inputs that come from the training and test sets through their representations. For the training inputs, we compute their representations and model their densities~\citep{Goodfellow-et-al-2016}. To this end, we leverage GMMs as universal approximators of densities. We give representations a probabilistic interpretation such that they have a smooth enough density which can be approximated by any specific nonzero amount of error using a GMM with enough components. Each component has a separately parameterized mean $\mu$ and covariance $\Sigma$. In some cases, we observe that GMMs can overfit to their training data when no constraints are applied to the covariance matrix, hence we limit the covariance matrix for each component to be diagonal. Moreover, this constraint makes training more computationally efficient since it avoids storing and inverting full high-dimensional covariance matrices. %

\subsection{Data Flow}
\label{sec:data-flow}

The full flow of our dataset inference for encoders consists of the following four main steps (which are also visualized in Figure~\ref{fig:dataset-inference-overview} in Appendix):
\begin{enumerate}[leftmargin=*]
    \item \textbf{Victim Training.} The victim's encoder $f_v$ is trained using the whole private training dataset $D_P$.
    \item \textbf{Encoder Stealing.} To steal the victim encoder $f_v$, an adversary queries $f_v$ with data points from $D_S$ to obtain representations $h_{D_S} \in \gR^n: h_{D_S} = f_v(D_S)$. 
    With these representations, the adversary trains the stolen encoder $f_s$ in a contrastive manner.
    \item \textbf{Training Estimators.} To perform ownership resolution, an arbitrator trains three density estimators $\gE_v$, $\gE_s$ and $\gE_i$ for the victim $f_v$, stolen $f_s$, and independent encoder $f_i$ as follows:
    \begin{enumerate}[label=\alph*)]
    \item $D_P$ \new{(where $D_P$ is not necessarily the whole private training dataset)} is divided into two non-overlapping subsets $D_{P1}$ and $D_{P2}$.
    While $D_{P2}$ serves as the base for training the density estimators, $D_{P1}$ is used to evaluate density estimates of the private training data vs \new{part of} the test data $D_N$. %
    \item For a given encoder $f \in \{f_v, f_s, f_i\}$, the arbitrator generates representations $h \in \gR^n: h_{D_{P2}}=f(D_{P2})$ on dataset $D_{P2}$.
    Training the density estimators on the respective representations yields the final density estimators $\gE_v, \gE_s$, and $\gE_i$.
    \end{enumerate}
    \item \textbf{Estimating Densities.} The arbitrator generates representations of $D_{P1}$ and \new{(a subset of)} $D_{N}$ with each encoder $f \in \{f_v, f_s, f_i\}$.
    Applying the respective density estimator $\gE \in \{\gE_v, \gE_s, \gE_i\}$ on the representations yields the \prob of each data point $x$ in the respective dataset: $ \forall x \in D: p(x) = \gE(f(x))$.
\end{enumerate}

\subsection{Ownership Resolution}

For an encoder $f$, we compute the \prob on $D_{P1}$ and $D_N$ as: $u_P \coloneqq \frac{1}{|D_{P1}|}\sum_{x \in D_{P1}} \gE(f(x))$ and $u_N \coloneqq \frac{1}{|D_{N}|} \sum_{x \in {D_{N}}} \gE(f(x))$. 
The density estimator $\gE$ measures the similarity between the distributions over the victim's representations of the training $D_P$ vs test data $D_N$. 
The intuition behind the setup is that if an encoder was trained on $D_P$, representations of $D_{P1}$ are much more similar to representations of $D_{P2}$, because the whole dataset $D_P$ was used for training the encoder, however, representations of $D_N$ differ from representation of $D_{P2}$ since $D_N$ was not used to train the encoder.
For a victim $f_v$ and a stolen encoder $f_s$, $u_P$ is significantly larger than $u_N$, whereas, for an independent encoder $f_i$, the values do not differ significantly.
Finally, we carry out a hypothesis test with the null hypothesis being: $H_0 \coloneqq u_P \le u_N$. If the null-hypothesis can be rejected (p-value $< 0.05$), \ie when the \prob for the training set $D_{P1}$ is higher than that for the test set $D_N$,
we can conclude that the tested model was stolen. On the other hand, if the null hypothesis cannot be rejected then the test is inconclusive and we cannot determine if a tested encoder was stolen or not.

\section{Encoder Similarity Scores}
\label{sec:theory}

Measuring the quality of stolen encoders allows us to assess attacks and defenses.
In standard supervised learning, the quality of a stolen model is evaluated using two main objectives, namely task accuracy, which is the model's performance on the test set,
and fidelity, which is the agreement in the predictions for a given task between the stolen and the victim model~\citep{fidelity}. One of the approaches to measure the quality of an extracted encoder is to use its outputs to train a downstream task and compute the accuracy of that task or fidelity (with respect to the outputs of the downstream task trained on the victim encoder). \nlp{For example, the NLP models are evaluated on a variety of GLUE downstream tasks~\citep{wang-etal-2018-glue}.} However, a single downstream task cannot adequately reflect the degree of similarity between encoders since it reduces their high dimensional embeddings to single label representations, which are confounded by choices of downstream data and training protocol. 
Instead, we propose two new metrics. Our first metric is an information-theoretic score based on mutual information~\citep{CoverBook,KozLeo87}. 
Our second metric is a cosine similarity score based on the representations returned by different encoders. These metrics correspond to the \textit{fidelity} metric in supervised learning.
The behavior of the two metrics differs in certain cases, for example, when used on obfuscated representations (\eg with shuffled elements) or with independent models, however, we find that the overall trend is similar. Moreover, the mutual information score is based on an approximation while the cosine similarity score is calculated exactly given representation vectors. 
 Often the effectiveness of defenses may be underestimated against low-quality stolen copies that haven't successfully stolen victim behavior. Our metrics help disentangle such effects and enable faithful evaluation of defenses.

\subsection{Mutual Information Score}
\label{sec:mutinf}
Our first approach to assessing the quality of a stolen encoder uses a score based on mutual information. 
We sample $N$ data points from the victim's private training dataset $D_P$ and pass them through the encoders $f_v, f_s$, and $f_i$ to generate the respective representations. Per standard practice, we recenter and normalize the representations~\citep{ding2021grounding}. We denote the entropy by $H$ and compute it according to Algorithm~\ref{alg:entropy-estimator} which takes $f_v$, $D_P$, and $N$ as input. For the joint entropy $H(f_v,f_s|D)$, we generate representations from the two encoders (in this case victim $f_v$ and stolen $f_s$) and concatenate them, which increases the dimensionality of the final representation to $2d$, while other steps remain unchanged. A detailed algorithm for computing the joint entropy can be found in the Appendix as Algorithm~\ref{alg:joint-entropy-estimator}. We compute an \textbf{approximate score} that is based on the definition of mutual information $I(f_v,f_s|D_P)$ between the victim encoder $f_v$ and the stolen copy $f_s$ as well as the analogous mutual information $I(f_v, f_i|D_P)$ between the victim encoder $f_v$ and the independently trained encoder $f_i$. 
We rely on approximations since we measure mutual information using finite data.
Yet, in practice, such approximations have proven useful~\citep{mcallester2020mutualinformation}.
We define our mutual information score as follows:
\begin{equation}\label{eq: mutinf}
    I(f_v,f_s|D_P) = H(f_v|D_P) + H(f_s|D_P) - H(f_v,f_s|D_P)\text{.}
\end{equation}

A higher value of the mutual information $I(f_v,f_s|D_P)$ indicates a higher information leakage incurred by the stolen encoder. 
Expectedly, mutual information is higher between the victim and the stolen encoder than between the victim and independent encoders $I(f_v,f_s|D_P) >> I(f_v,f_i|D_P)$. We can normalize mutual information into a score (between 0 and 1) by setting the lower bound as the mutual information between the victim $f_v$ and a randomly initialized model $f_r$: $I_{min} = I(f_v,f_r|D_P)$ and the upper bound as the mutual information between the victim and itself: $I_{max} = I(f_v,f_v|D_P)$ %
For the current mutual information score $I_c$, the normalized score is defined as $S := \frac{I_c - I_{min}}{I_{max} - I_{min}}$.

\subsection{Cosine Similarity Score} %

The second score we use to assess the quality of a stolen encoder is based on the cosine similarity between its representations and the victim's representations. 
More specifically, we first compute representations for the two encoders on a set of $N$ randomly selected data points from the dataset $D_P$. %
Again as per standard practice~\cite{ding2021grounding}, we recenter and normalize these representations.
For each of the $N$ inputs, we then compute %
the cosine similarity between the corresponding representations from both encoders where the cosine similarity $\mathrm{sim}(a,b) = \frac{a^Tb}{||a||_2||b||_2}$ for representation vectors $a$ and $b$. %

We show %
(in Section~\ref{sec:analysis}) that the loss functions, which are used for stealing encoders, directly maximize the cosine similarity between representations from victim and stolen encoders. %
We thus propose %
to use the cosine similarity score $C$ as a metric, which we define as:  $C = |\mathrm{sim}(a,b)|$ (2).
The score yields values in the range $[0,1]$%
, with a higher score indicating closer representations. To calculate a per-encoder cosine similarity score, we average the cosine similarity scores %
over all inputs. %
We find that the cosine similarity score is well-calibrated %
across encoders. Namely, an independent encoder, expected to have representations unrelated to the victim encoder, has an average cosine similarity concentrated around $0$~\citep{SSLguard}, %
while a stolen encoder exhibits significantly higher scores. %
The cosine similarity score is also easy to compute since it only requires the corresponding representations of the two models and their dot product.

\subsection{Analysis}
\label{sec:analysis}
There are various ways in which an attacker may steal an encoder. To simplify our analysis of the cosine similarity score, we consider the two best-performing loss functions used for stealing~\citep{SSLextraction}: %
the first where the attacker minimizes the non-contrastive MSE (Mean Squared Error) loss between its representations and the victim encoder's representations to train the stolen encoder, %
and the second where the attacker uses a contrastive loss function, such as the InfoNCE loss~\citep{cpc} which is used in SimCLR~\citep{simclr}. %

\textbf{Stealing with MSE loss.} In the case where the MSE loss is used, %
let $x_i$ be a query made by an attacker and let $f_v(x_i)  = h_{v_i}, f_s(x_i) = h_{s_i} \in \mathcal{R}^n$ be the corresponding representations of the victim and stolen encoders, respectively. The MSE loss between these two representations is $\frac{1}{n} \sum_{j=1}^{n} (h_{{v_i}_j} - h_{{s_i}_j})^2 = \frac{1}{n} ||h_{v_i} - h_{s_i}||_2^2$. It follows directly that minimizing the MSE loss also minimizes the $\ell_2$ distance between representations and equivalently maximizes the cosine similarity between representations:
\textbf{Theorem 1} $||a-b||_2 = \sqrt{2(1 - \mathrm{sim} (a,b))}, ||a||_2=||b||_2=1$ 
(see~\ref{sec:l2analysis}).

\textbf{Stealing with a contrastive loss.} When an attacker uses a contrastive loss function %
for stealing%
, minimizing the loss corresponds to maximizing the sum of the cosine similarities between positive pairs, \ie $\sum_{c=1}^{m} (\mathrm{sim} (h_{s_c}, h_{v_c}) / \tau)$.
The InfoNCE loss, or contrastive losses in general, also increase the mutual information score~\citep{mutualmax,cpc}.
We therefore expect that stolen encoders will have larger similarity scores w.r.t. the victim encoder than independent encoders. %
We refer the reader to Appendix~\ref{sec:l2analysis} for a more detailed discussion of the loss functions and their relationship with the similarity scores. %

\section{Empirical Evaluation}
\label{sec:eval}

We evaluate our defense against encoder extraction attacks using five different vision datasets (CIFAR10, CIFAR100~\citep{Krizhevsky09learningmultiple}, SVHN~\citep{svhn}, STL10~\citep{coates2011stl10}, and ImageNet~\citep{deng2009imagenet}). 
Table \ref{tab:data-inference-gmm} shows that our dataset inference method is able to differentiate between the stolen copies of the victim encoder and independently trained encoders by using the victim's private training data as the signature.
We also show that our defense works in the scenario where the adversary modifies the representations to render them inconspicuous, \eg by shuffling the order of elements in the representation vectors. To assess the quality of the stolen encoders and the performance of our defense, we measure the mutual information and cosine similarity scores between encoders %
and present our results in Tables~\ref{tab:similarity_scores} and~\ref{tab:num-queries}. %

\subsection{Training Victim, Stolen and Independent Encoders}

\textbf{Victim.} We use victim encoders trained on the ImageNet,
CIFAR10, and SVHN datatsets. For the ImageNet victim encoder %
, we use a %
model released by the authors of SimSiam~\citep{SimSiam}. To train CIFAR10 and SVHN victim encoders, %
we use an open-source PyTorch implementation of SimCLR \footnote{\url{https://github.com/kuangliu/pytorch-cifar}}. For SVHN, we merge the original training and test splits, and use the randomly-selected 80$\%$ as the training set and the rest 20$\%$ as the test set. This is necessary because the original training and test splits for SVHN are not i.i.d \citep{rabanser2019failing}, which violates the assumption for dataset inference (see Section \ref{sec:protocol}).  The ImageNet victim has an output representation dimension of 2048, while the CIFAR10 and SVHN victim encoders have 512-dimensional representations. 

\textbf{Stolen.} When stealing from the victim encoders, we evaluate different numbers of queries from various datasets, including CIFAR10, SVHN, ImageNet, and STL10. %
Stolen encoders are trained in a similar contrastive way as the victim and use the InfoNCE loss, %
where the positive pairs consist of representations from the victim and stolen encoder for a given input. Algorithm~\ref{alg:stealing} summarizes the stealing approach used by an adversary.

\textbf{Independent.} %
For each victim encoder, we train independent encoders using datasets different from the victim's private training dataset $D_P$. The encoders are trained with the SimCLR approach, similar to the way the victim encoders were trained. In the case where the dataset used to train the independent model had different image dimensions from the victim's training dataset, the dataset was resized to be of the same size.

More details on the training and stealing of encoders can be found in Section~\ref{sec:detailsexp} of the Appendix.

\subsection{Dataset Inference on Encoders}

\textbf{Setup.} We train GMMs with $10$ components for SVHN and CIFAR10, and $50$ components for ImageNet. 
\new{In general, we observe that the larger number of components for GMMs, the better the defense is.}
For ImageNet, we restrict the covariance matrix to be diagonal for efficiency. For CIFAR10 and SVHN, we use the full covariance matrix. For SVHN and CIFAR10, we use $50 \%$ of the training set to train GMMs, and the remaining for evaluation. For ImageNet, we use $100K$ images from the training set to train GMMs, and another $100K$ of the training set as an evaluation set. We normalize representations by $l_2$ norm for training GMM. For ImageNet, we also standardize representations (subtract mean and divide by standard deviation) before normalization. We do not use augmentations in dataset inference. For each setting, the hyperparameters are tuned on the victim model and a randomly-initialized model.%

\textbf{Evaluation of our Defense.} The empirical results in Table \ref{tab:data-inference-gmm} demonstrate that we are able to differentiate between stolen and independent encoders from the difference in log-likelihoods. We observe that the stolen encoders have significantly larger $\Delta \mu$ than the independent encoders. The p-values further show that for stolen encoders the null hypothesis is rejected while for independent encoders, the test is inconclusive. Similar to dataset inference for supervised learning~\citep{maini2021dataset}, the victim model typically has the largest $\Delta \mu$ and the smallest p-values. We also observe that our method is better at detecting encoders that are stolen using queries from the victim's training set.%

\textbf{Number of Stolen Queries.} Table \ref{tab:num-queries} shows that as the attacker steals with more queries, the p-value from our defense becomes lower. This is consistent with the finding in \citep{maini2021dataset} that dataset inference works better with stronger stolen encoders. We also find that our defense is able to detect stolen encoders even if the attacker only steals from a small number of queries. For example, in Table \ref{tab:num-queries}, we are able to claim ownership when only 50K - 100k queries are used for stealing ImageNet victims (around 4$\%$ of its training set).

\textbf{Robustness of Dataset Inference to Obfuscations.} The attacker can obfuscate the stolen encoder representations by, for instance, applying shuffling (changing the order of elements), padding (adding zeros), or linear transformations (\eg scaling or adding a constant). These obfuscations have little impact on the downstream performance~\citep{haochen2021provable} but may pose challenges to the defenses of the victim. %
The results in Table~\ref{tab:data-inference-gmm} show that the p-values for the stolen encoders after attackers' obfuscations remain low, which implies that our method is robust to these types of obfuscations. %

\begin{table}[t]
\caption{
\textbf{
Dataset inference via density estimation of representations.} 
We detect if a given encoder was stolen.
$f_v$ denotes the victim encoder trained on data $D$, $f_s$ is the stolen encoder extracted using queries from a given stealing dataset $D$, and $f_i$ is an independent encoder trained on data $D$ (different than the victim's private training data). Each value is an average of 3 trials. $\Delta \mu$ is the effect size from the statistical t-test. Obfuscations: the representation can be modified by an attacker in the following ways: (1) \textit{Shuffle} the elements in the representation vectors, (2) \textit{Pad} with zeros or add zeros at random positions, and (3) apply a linear \textit{Transform}. The first row below denotes the victim's private data $D_P$.
}
\label{tab:data-inference-gmm}
\scriptsize
    \centering
    \begin{tabular}{ccccccccccc}
    \toprule
         \multicolumn{2}{c}{\textit{\textbf{Victim's private data:}}} & & \multicolumn{2}{c}{\textit{\textbf{CIFAR10}}} && \multicolumn{2}{c}{\textit{\textbf{SVHN}}} ~ & ~ & \multicolumn{2}{c}{\textit{\textbf{ImageNet}}} \\ 
         \cmidrule{1-2} \cmidrule{4-5} \cmidrule{7-8} \cmidrule{10-11} 
         Encoder & Obfuscate & $D$ & p-value & $\Delta \mu$ &
         $D$ & p-value & $\Delta \mu$ &
         $D$ & p-value & $\Delta \mu$ \\
        \midrule
        $f_v$ & N/A & CIFAR10 & \new{5.61e-82} & \new{18.92} & SVHN & 2.75e-125 & 23.88  & ImageNet & 6.23e-14 & 7.09 \\ 
        \cdashlinelr{1-11}
        \multirow{3}{*}{$f_s$} & \multirow{3}{*}{N/A} & SVHN & \new{3.97e-2} & \new{3.04} & SVHN & \new{6.35e-41} & \new{13.36} & SVHN & 3.33e-4 & 4.04 \\
        & & CIFAR10 & \new{8.73e-7} & \new{5.09} & CIFAR10 & 2.38e-4 & 4.61 & CIFAR10 & 1.47e-4  & 6.21 \\
        & & STL10 & \new{1.04e-2} & \new{3.42} & STL10 & 1.23e-5 & 5.22  & \new{STL10} & \new{1.09e-4} & 5.87 \\ 
        & & \new{ImageNet} & \new{6.34e-3}  & \new{3.47} & \new{ImageNet} & \new{9.81e-3} & \new{3.74} & ImageNet & 3.14e-5 & 7.32 \\
        \cdashlinelr{1-11}
        \multirow{3}{*}{$f_s$} & \textit{Shuffle} & CIFAR10 & \new{1.72e-6} & \new{4.98} & CIFAR10 & 7.32e-4  & 4.77 & CIFAR10 & 6.72e-4 & 5.21 \\
        & \textit{Pad} & CIFAR10 & \new{3.44e-6} & \new{4.84} & CIFAR10 & 2.51e-3 & 3.08 & CIFAR10 & 2.31e-3 & 4.23 \\
        & \textit{Transform} & CIFAR10 & \new{6.81e-7} & \new{5.11} & CIFAR10 & 6.45e-3 & 3.32 & CIFAR10 & 8.45e-3 & 3.98 \\
        \cdashlinelr{1-11}
        \multirow{2}{*}{$f_i$} & \multirow{2}{*}{N/A} & CIFAR100 & \new{3.67e-1} & \new{-0.37} & CIFAR100 & 6.21e-1 & 0.52  & CIFAR100 & 7.53e-2 & 1.63 \\
        & & SVHN & \new{2.96e-1} & \new{0.98} & CIFAR10 & 4.82e-1 & 0.56 & SVHN & 5.42e-1 & 0.69 \\
        \bottomrule
    \end{tabular}
    \vskip -0.1in
\end{table}

\subsection{Measuring Quality of Stolen Encoders} \label{sec:quality}

\newcommand\xdensity{0.4} %
\newcommand\vskipdensity{-0.2in}
\begin{wrapfigure}{r}{\xdensity\textwidth}
\vskip \vskipdensity
  \begin{center}
    \includegraphics[width=\xdensity\textwidth]{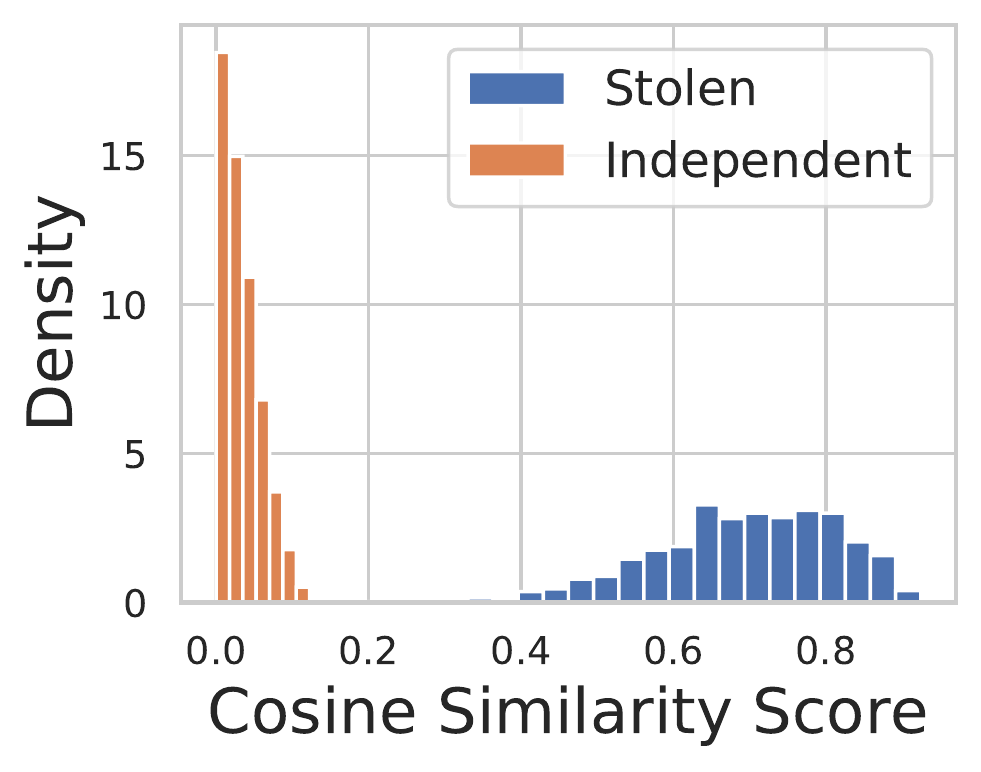}
  \end{center}
  \caption{
  Distribution of \textbf{cosine similarity scores}. 
  \label{fig:cosinedist}
  }
\vskip \vskipdensity
\end{wrapfigure}

\textbf{Setup.} To measure the quality of stolen encoders, we select a random subset of $N=20K$ unaugmented images from the private training dataset $D_P$ and compute their representations from stolen and victim encoders. We then centralize (subtract the mean for each dimension) and normalize the representations (divide by the $\ell_2$ norm). For the mutual information score, we first estimate the entropies $H(f_v), H(f_s), H(f_v,f_s)$, which are then added and normalized as in Section \ref{sec:mutinf}. The score is capped to be in the range $[0, 1]$.
To compute the cosine similarity score, %
we find the absolute value of the dot product of corresponding representations for the two encoders 
(Equation 2).
These dot products are then averaged over all representations. %

\textbf{Evaluation of Metrics.} To evaluate the mutual information and the cosine similarity scores, we conduct two sets of experiments to verify if: (1) the scores are higher for stolen than independent encoders, and (2) the scores increase as more queries are used to steal encoders, which suggests a higher quality of the stolen copies~\citep{SSLextraction}. %
In Table~\ref{tab:similarity_scores}, we observe that both our scores assign higher values to the stolen encoders than the independent encoders. Table~\ref{tab:num-queries} shows that our mutual information and cosine similarity scores generally increase while the p-values from our dataset inference decrease with respect to the number of queries used to steal an encoder. This implies that the performance of our defense is consistent with the similarity metrics and becomes more effective as the quality of the stolen encoder improves.
We also plot a histogram of the cosine similarity scores for the stolen and independent encoders in Figure~\ref{fig:cosinedist} for an SVHN victim encoder, a stolen encoder from it (using CIFAR10 training data for queries), and an independent encoder (trained on CIFAR100). 
There is a pronounced difference between the two distributions with the cosine similarity scores for the independent encoder being close to $0$ and the scores for the stolen encoder being much higher than $0$. 

\textbf{Robustness of Metrics to Obfuscations.} We also consider the effect of obfuscations on these metrics. Without any obfuscation of the representations from stolen encoders, the cosine similarity score shows a clearer distinction between stolen and independent encoders than the mutual information score: in Table~\ref{tab:similarity_scores}, the cosine similarity scores for all independent encoders are close to zero, but the mutual information scores can be quite high (such as $0.9$ for the independent encoders of SVHN, which is likely because of the mutual information score being based on an approximation). 
However, the mutual information score is robust to the obfuscations of the attackers while cosine similarity is not: in Table~\ref{tab:similarity_scores}, the cosine similarity score for the stolen encoders after shuffling and padding drops close to zero. Mutual information, as a more general metric based on the information measurement instead of the brittle structure of the representation vectors, performs better and is oblivious to the obfuscations that attackers might introduce.

\begin{table}[t]
\caption{
\label{tab:similarity_scores}
\textbf{Encoder similarity scores.} %
We compare encoders via the encoder quality metrics using the same setting as in Table~\ref{tab:data-inference-gmm}. We compute the score $S(\cdot,f_v)$ based on the mutual information between a given encoder (in a row) and the victim encoder $f_v$. Analogously, we compute the cosine similarity score $C(\cdot, f_v)$. 
}
\tiny%
    \centering
     \begin{tabular}{p{2em}cx{4.5em}p{3em}cx{4.5em}ccx{4.5em}cc}
    \toprule
        \multicolumn{2}{c}{\textit{\textbf{Victim's private data:}}} & & \multicolumn{2}{c}{\textit{\textbf{CIFAR10}}} && \multicolumn{2}{c}{\textit{\textbf{SVHN}}} ~ & ~ & \multicolumn{2}{c}{\textit{\textbf{ImageNet}}} \\ 
         \cmidrule{1-2} \cmidrule{4-5} \cmidrule{7-8} \cmidrule{10-11} 
         Encoder & Obfuscate & $D$ & $S(\cdot,f_v)$ & $C(\cdot,f_v)$ &
         $D$ & $S(\cdot,f_v)$ & $C(\cdot,f_v)$ &
         $D$ & $S(\cdot,f_v)$ & $C(\cdot,f_v)$ \\
        \midrule
        $f_v$ & N/A & CIFAR10 & 1.0 & 1.0 & SVHN & 1.0 & 1.0  & ImageNet &  1.0 & 1.0 \\ 
        \cdashlinelr{1-11}
        \multirow{3}{*}{$f_s$} & N/A & SVHN & 0.73 & 0.504 & SVHN & 0.96 & 0.91 & SVHN & 0.86 & 0.39  \\
        & N/A & CIFAR10 & 0.84 & 0.95 & CIFAR10 & 0.94 & 0.69 & CIFAR10 & 0.88 & 0.43 \\
        &N/A &  STL10 & 0.89 & 0.92 & STL10 & 0.95 & 0.89 & ImageNet & 0.96 & 0.78 \\
        \cdashlinelr{1-11}
        \multirow{3}{*}{$f_s$} & \textit{Shuffle} & SVHN & 0.74 & 0.002 & CIFAR10 & 0.94 & 0.003 & SVHN & 0.86 & 0.005 \\
        & \textit{Pad} & SVHN & 0.74 & 0.007 & CIFAR10 & 0.93 & 0.013 & SVHN & 0.85 & 0.003 \\
        & \textit{Transform} & SVHN & 0.75 & 0.504 & CIFAR10 & 0.93 & 0.69 & SVHN & 0.86 & 0.39 \\
        \cdashlinelr{1-11}
        \multirow{2}{*}{$f_i$} & N/A & CIFAR100 & 0.63 & 0.0007 & CIFAR100 & 0.90 & 0.007 & CIFAR100 &  0.81 & 0.0018 \\
        & N/A & SVHN & 0.12 & 0.0001 & CIFAR10 & 0.90 & 0.009 & SVHN & 0.75 & 0.002 \\
        \bottomrule
    \end{tabular}
    \vskip -0.1in
\end{table}

\begin{table}[t]
\caption{
\textbf{Encoder similarity scores and p-values from dataset inference vs the number of queries}. The quality of the stolen encoders increases with more stealing queries, which is reflected by the rise in the mutual information and cosine similarity scores as well as the better performance of our defense as indicated by the decreasing p-values. $D_P$ is the private dataset used to train the victim and $D_S$ is the dataset used for stealing.
}
\label{tab:num-queries}
\tiny %
    \centering
      \begin{tabular}{cccccccccccc}
    \toprule
         $D_P$ & $D_S$ & \textbf{Score} &\multicolumn{9}{c}{\textbf{Number of Queries}} \\
         & & & 5K & 10K & \new{20K} & \new{30K} & \new{40K} & 50K & 100K & 200K & 250K \\
        \cmidrule{4-12} 
        \multirow{3}{*}{ImageNet} & \multirow{3}{*}{SVHN} & $S(\cdot,f_v)$ & 0.62 & 0.79 & \new{0.79} & \new{0.81} & \new{0.82} & 0.84 & 0.85 & 0.85 & 0.86 \\
        & & $C(\cdot,f_v)$ & 0.25 & 0.32 & \new{0.33} & \new{0.36} & \new{0.35} & 0.38 & 0.38 & 0.40 & 0.39 \\
        & & p-values & 1.23e-1 & 7.91e-2 & \new{6.53e-2} & \new{8.98e-2} & \new{4.52e-2} & 1.10e-2 &2.11e-3 & 1.11e-3 & 3.33e-4 \\
        \cdashlinelr{1-12}
         & & & 500 & 5K & 7K & 8K & 9K & 10K & 30K & 40K & 50K \\
        \cmidrule{4-12} 
        \multirow{3}{*}{ImageNet} & \multirow{3}{*}{CIFAR10} & $S(\cdot,f_v)$ & 0.55 & 0.60 & \new{0.62} & \new{0.75} & \new{0.58} & 0.64 & 0.87 & 0.82 & 0.88 \\
        & & $C(\cdot,f_v)$ & 0.21 & 0.28 & \new{0.31} & \new{0.29} & \new{0.36} & 0.32 & 0.40 & 0.41 & 0.43 \\
        & & p-values & 8.88e-2 & 7.12e-2 & 8.23e-1 & 4.14e-1 & 3.41e-3 & 8.51e-21 & 9.23e-2 & 7.32e-2 & 1.47e-4 \\
        \cdashlinelr{1-12}
         & & & 500 & 5K & 7K & 8K & 9K & 10K & 30K & 50K & 100K \\
        \cmidrule{4-12} 
        \multirow{3}{*}{ImageNet} & \multirow{3}{*}{STL10} & $S(\cdot,f_v)$ & 0.76 & 0.75 & \new{0.72} & \new{0.81} & \new{0.84} & 0.81 & 0.89 & 0.88 & 0.92 \\
        & & $C(\cdot,f_v)$ & 0.28 & 0.29 & 0.36 & 0.38 & 0.37 & 0.43 & 0.44 & 0.52 & 0.58 \\
        & & p-values & 9.63e-1 & 8.21e-1 & 7.32e-1 & 5.44e-1 & 1.21e-1 & 5.98e-2 & 8.11e-2 & 6.28e-2 & 1.09e-4 \\
        \cdashlinelr{1-12}
         & & & 5K & 10K & \new{20K} & \new{30K} & \new{40K} & 50K & 100K & 200K & 250K \\
        \cmidrule{4-12} 
        \multirow{3}{*}{ImageNet} & \multirow{3}{*}{ImageNet} & $S(\cdot,f_v)$ & 0.61 & 0.75 & 0.73 & 0.76 & 0.81 & 0.91 & 0.90 & 0.95 & 0.96 \\
        & & $C(\cdot,f_v)$ & 0.29  & 0.48 & 0.49 & 0.51 & 0.46 & 0.38 & 0.52  & 0.76   & 0.78  \\
        & & p-values & 9.88e-1 & 3.21e-1 & \new{5.32e-1} & \new{1.08e-1} & \new{3.61e-3} & 3.97e-4 & 5.34e-4 & 8.72e-4 & 3.14e-5 \\       
   \bottomrule
    \end{tabular}
     \vskip -0.1in
\end{table}

\nlp{
\subsection{NLP Models}
In our experiment with language models, we focused on models using the TinyBERT~\cite{jiao2019tinybert} architecture.
\textbf{Training.} We trained our victim model $\gV$ from scratch using the English Wikipedia~\cite{wikidump} dataset and a pre-trained BERT~\cite{devlin2018bert} tokenizer. \yannis{We use a masking probability of 0.15.}
\textbf{Stealing.} Model stealing was performed using the English subset of the OSCAR dataset~\cite{oscardataset}.
By obtaining representations from~$\gV$, the stolen model was trained using the MSE-loss.
\textbf{Validation.}
We additionally trained an independent model~$\gI$ on a BookCorpus dataset~\cite{Zhu_2015_ICCV}.
We tokenize the training data of both~$\gV$ and~$\gI$ using a pre-trained BERT-tokenizer.
\textbf{Evaluation.} We evaluate all models, $\gV$,~$\gI$ and~$\gS$ on the GLUE benchmark~\cite{wang2019glue}.
\textbf{Dataset Inference.}
We then run the dataset inference between the victim and the stolen model and check if the p-value is $ \leq 0.01$.
A a checkup, we also run dataset inference between the victim and the independent random model and check if the p-value is $> 0.01$.
}

\subsection{Limitations}
\label{sec:limitations}
\new{
If the t-test run as part of dataset inference is inconclusive for an extracted encoder, we cannot state whether the encoder was stolen. Similarly, for an independent encoder, there is the possibility of it being incorrectly classified as stolen. Previous work~\cite{Kotar_2021_ICCV, concept_generalization_2021_ICCV} has shown that self-supervised encoders trained using heavy augmentations and contrastive learning generalize better than their supervised counterparts, which makes it harder for the dataset inference to differentiate between train and test representations in SSL than in the SL setting~\cite{SSLextraction}.
The loss values of projected individual representations are insufficient for dataset inference~\cite{SSLextraction}. We build on top of this observation to enable dataset inference for encoders and use GMMs to distinguish between
train and test representations.
}
\section{Conclusions}

New public APIs expose self-supervised encoder models which return high-dimensional embeddings for provided inputs. Adversaries can use these embeddings to steal the encoders. We present a novel method based on dataset inference for defending against such stealing attacks along with metrics to assess the quality of the stolen encoders and to quantify the effectiveness of our defense. %
We observe that knowledge contained in the private training set is transferred from the victim encoder to its stolen copy. Thus, the private data acts as a signature of the victim encoder. 
By leveraging density estimation on the respective encoders' representations, we obtain a signal allowing us to differentiate between the encoder's training and test data. 
This difference is detectable in both the victim encoder and its stolen copy but not in independent encoders which are legitimately trained on different data than the victim's private training data.
Thus, we are able to flag the stolen copy of the victim encoder while not accusing creators of legitimately trained encoders of theft. We show the high effectiveness of our defense on vision encoders. 
Future work may explore additional applications of our proposed defense and metrics beyond model stealing and ownership verification, as well as their use in other domains such as natural language processing (NLP).
In particular, our method may help enforce the ethical usage of sensitive online data, such as images on social media, in accordance with privacy regulations by auditing if a given provider's encoder contains knowledge of these sensitive data.

\section*{Acknowledgments}
We would like to acknowledge our sponsors, who support our research with financial and in-kind contributions: CIFAR through the Canada CIFAR AI Chair program, DARPA through the GARD program, Intel, Meta, NFRF through an Exploration grant, and NSERC through the Discovery Grant and COHESA Strategic Alliance. Resources used in preparing this research were provided, in part, by the Province of Ontario, the Government of Canada through CIFAR, and companies sponsoring the Vector Institute. We would like to thank members of the CleverHans Lab for their feedback.

\bibliographystyle{plainnat}
\bibliography{main}

\begin{thebibliography}{45}
\providecommand{\natexlab}[1]{#1}
\providecommand{\url}[1]{\texttt{#1}}
\expandafter\ifx\csname urlstyle\endcsname\relax
  \providecommand{\doi}[1]{doi: #1}\else
  \providecommand{\doi}{doi: \begingroup \urlstyle{rm}\Url}\fi

\bibitem[Coh()]{Cohere}
Cohere, https://cohere.ai.
\newblock URL \url{https://cohere.ai/}.

\bibitem[Ope()]{OpenAI}
Openai, https://openai.com.
\newblock URL \url{https://openai.com/}.

\bibitem[Bachman et~al.(2019)Bachman, Hjelm, and Buchwalter]{mutualmax}
Philip Bachman, R~Devon Hjelm, and William Buchwalter.
\newblock Learning representations by maximizing mutual information across
  views, 2019.
\newblock URL \url{https://arxiv.org/abs/1906.00910}.

\bibitem[Boenisch(2021)]{watermarks2021}
Franziska Boenisch.
\newblock A systematic review on model watermarking for neural networks.
\newblock \emph{Frontiers in Big Data}, 4, nov 2021.
\newblock \doi{10.3389/fdata.2021.729663}.

\bibitem[Chen et~al.(2020)Chen, Kornblith, Norouzi, and Hinton.]{simclr}
Ting Chen, S.~Kornblith, M.~Norouzi, and G.~Hinton.
\newblock A simple framework for contrastive learning of visual
  representations.
\newblock \emph{International Conference on Machine Learning}, 2020.

\bibitem[Chen and He(2020)]{SimSiam}
Xinlei Chen and Kaiming He.
\newblock Exploring simple siamese representation learning.
\newblock 2020.

\bibitem[Chen et~al.(2021)Chen, Wang, Bender, Ding, Jia, Li, and
  Song]{chen2021refit}
Xinyun Chen, Wenxiao Wang, Chris Bender, Yiming Ding, Ruoxi Jia, Bo~Li, and
  Dawn Song.
\newblock Refit: a unified watermark removal framework for deep learning
  systems with limited data.
\newblock In \emph{Proceedings of the 2021 ACM Asia Conference on Computer and
  Communications Security}, pages 321--335, 2021.

\bibitem[Coates et~al.(2011)Coates, Ng, and Lee]{coates2011stl10}
Adam Coates, Andrew Ng, and Honglak Lee.
\newblock {An Analysis of Single Layer Networks in Unsupervised Feature
  Learning}.
\newblock In \emph{AISTATS}, 2011.
\newblock
  \url{https://cs.stanford.edu/~acoates/papers/coatesleeng_aistats_2011.pdf}.

\bibitem[Cong et~al.(2022)Cong, He, and Zhang]{SSLguard}
Tianshuo Cong, Xinlei He, and Yang Zhang.
\newblock Sslguard: {A} watermarking scheme for self-supervised learning
  pre-trained encoders.
\newblock \emph{CoRR}, abs/2201.11692, 2022.
\newblock URL \url{https://arxiv.org/abs/2201.11692}.

\bibitem[Cover and Thomas(2006)]{CoverBook}
Thomas~M. Cover and Joy~A. Thomas.
\newblock \emph{Elements of Information Theory (Wiley Series in
  Telecommunications and Signal Processing)}.
\newblock Wiley-Interscience, USA, 2006.
\newblock ISBN 0471241954.

\bibitem[Deng et~al.(2009)Deng, Dong, Socher, Li, Li, and
  Fei-Fei]{deng2009imagenet}
Jia Deng, Wei Dong, Richard Socher, Li-Jia Li, Kai Li, and Li~Fei-Fei.
\newblock Imagenet: A large-scale hierarchical image database.
\newblock In \emph{2009 IEEE conference on computer vision and pattern
  recognition}, pages 248--255. Ieee, 2009.

\bibitem[Devlin et~al.(2018)Devlin, Chang, Lee, and Toutanova]{devlin2018bert}
Jacob Devlin, Ming-Wei Chang, Kenton Lee, and Kristina Toutanova.
\newblock Bert: Pre-training of deep bidirectional transformers for language
  understanding.
\newblock \emph{arXiv preprint arXiv:1810.04805}, 2018.

\bibitem[Ding et~al.(2021)Ding, Denain, and Steinhardt]{ding2021grounding}
Frances Ding, Jean-Stanislas Denain, and Jacob Steinhardt.
\newblock Grounding representation similarity through statistical testing.
\newblock In A.~Beygelzimer, Y.~Dauphin, P.~Liang, and J.~Wortman Vaughan,
  editors, \emph{Advances in Neural Information Processing Systems}, 2021.
\newblock URL \url{https://openreview.net/forum?id=_kwj6V53ZqB}.

\bibitem[Dziedzic et~al.(2022{\natexlab{a}})Dziedzic, Dhawan, Kaleem, Guan, and
  Papernot]{SSLextraction}
Adam Dziedzic, Nikita Dhawan, Muhammad~Ahmad Kaleem, Jonas Guan, and Nicolas
  Papernot.
\newblock On the difficulty of defending self-supervised learning against model
  extraction.
\newblock In \emph{International Conference on Machine Learning},
  2022{\natexlab{a}}.

\bibitem[Dziedzic et~al.(2022{\natexlab{b}})Dziedzic, Kaleem, Lu, and
  Papernot]{powDefense}
Adam Dziedzic, Muhammad~Ahmad Kaleem, Yu~Shen Lu, and Nicolas Papernot.
\newblock Increasing the cost of model extraction with calibrated proof of
  work.
\newblock In \emph{International Conference on Learning Representations},
  2022{\natexlab{b}}.
\newblock URL \url{https://arxiv.org/abs/2201.09243}.

\bibitem[Goodfellow et~al.(2016)Goodfellow, Bengio, and
  Courville]{Goodfellow-et-al-2016}
Ian Goodfellow, Yoshua Bengio, and Aaron Courville.
\newblock \emph{Deep Learning}.
\newblock MIT Press, 2016.
\newblock \url{http://www.deeplearningbook.org}.

\bibitem[HaoChen et~al.(2021)HaoChen, Wei, Gaidon, and Ma]{haochen2021provable}
Jeff~Z HaoChen, Colin Wei, Adrien Gaidon, and Tengyu Ma.
\newblock Provable guarantees for self-supervised deep learning with spectral
  contrastive loss.
\newblock \emph{Advances in Neural Information Processing Systems}, 34, 2021.

\bibitem[He and Zhang(2021)]{membershipInferenceSSL2021}
Xinlei He and Yang Zhang.
\newblock Quantifying and mitigating privacy risks of contrastive learning.
\newblock In \emph{Proceedings of the 2021 ACM SIGSAC Conference on Computer
  and Communications Security}, CCS '21, page 845–863, New York, NY, USA,
  2021. Association for Computing Machinery.
\newblock ISBN 9781450384544.
\newblock \doi{10.1145/3460120.3484571}.
\newblock URL \url{https://doi.org/10.1145/3460120.3484571}.

\bibitem[Jagielski et~al.(2020)Jagielski, Carlini, Berthelot, Kurakin, and
  Papernot.]{fidelity}
Matthew Jagielski, N.~Carlini, D.~Berthelot, A.~Kurakin, and N.~Papernot.
\newblock High accuracy and high fidelity extraction of neural networks.
\newblock \emph{USENIX Security Symposium}, 2020.

\bibitem[Jaiswal et~al.(2021)Jaiswal, Babu, Zadeh, Banerjee, and
  Makedon]{ssl-survey2021}
Ashish Jaiswal, Ashwin~Ramesh Babu, Mohammad~Zaki Zadeh, Debapriya Banerjee,
  and Fillia Makedon.
\newblock A survey on contrastive self-supervised learning.
\newblock \emph{Technologies}, 9\penalty0 (1), 2021.
\newblock ISSN 2227-7080.
\newblock \doi{10.3390/technologies9010002}.
\newblock URL \url{https://www.mdpi.com/2227-7080/9/1/2}.

\bibitem[Jia et~al.(2021{\natexlab{a}})Jia, Choquette-Choo, Chandrasekaran, and
  Papernot.]{entangled_wm}
Hengrui Jia, C.~A. Choquette-Choo, V.~Chandrasekaran, and N.~Papernot.
\newblock Entangled watermarks as a defense against model extraction.
\newblock \emph{USENIX Security Symposium}, 2021{\natexlab{a}}.

\bibitem[Jia et~al.(2021{\natexlab{b}})Jia, Yaghini, Choquette-Choo, Dullerud,
  Thudi, Chandrasekaran, and Papernot]{jia2021proof}
Hengrui Jia, Mohammad Yaghini, Christopher~A Choquette-Choo, Natalie Dullerud,
  Anvith Thudi, Varun Chandrasekaran, and Nicolas Papernot.
\newblock Proof-of-learning: Definitions and practice.
\newblock \emph{arXiv preprint arXiv:2103.05633}, 2021{\natexlab{b}}.

\bibitem[Juuti et~al.(2019)Juuti, Szyller, Marchal, and Asokan]{juuti2019prada}
Mika Juuti, Sebastian Szyller, Samuel Marchal, and N~Asokan.
\newblock Prada: protecting against dnn model stealing attacks.
\newblock In \emph{2019 IEEE European Symposium on Security and Privacy
  (EuroS\&P)}, pages 512--527. IEEE, 2019.

\bibitem[Kapoor et~al.(2007)Kapoor, Horvitz, and Basu]{kapoor2007selective}
Ashish Kapoor, Eric Horvitz, and Sumit Basu.
\newblock Selective supervision: Guiding supervised learning with
  decision-theoretic active learning.
\newblock In \emph{IJCAI}, volume~7, pages 877--882, 2007.

\bibitem[Kornblith et~al.(2019)Kornblith, Norouzi, Lee, and
  Hinton]{similaritynn}
Simon Kornblith, Mohammad Norouzi, Honglak Lee, and Geoffrey Hinton.
\newblock Similarity of neural network representations revisited, 2019.
\newblock URL \url{https://arxiv.org/abs/1905.00414}.

\bibitem[Kotar et~al.(2021)Kotar, Ilharco, Schmidt, Ehsani, and
  Mottaghi]{Kotar_2021_ICCV}
Klemen Kotar, Gabriel Ilharco, Ludwig Schmidt, Kiana Ehsani, and Roozbeh
  Mottaghi.
\newblock Contrasting contrastive self-supervised representation learning
  pipelines.
\newblock In \emph{Proceedings of the IEEE/CVF International Conference on
  Computer Vision (ICCV)}, pages 9949--9959, October 2021.

\bibitem[Kozachenko and Leonenko(1987)]{KozLeo87}
L.~F. Kozachenko and N.~N. Leonenko.
\newblock Sample estimate of the entropy of a~random vector.
\newblock \emph{Probl. Peredachi Inf.}, 23:\penalty0 9--16, 1987.
\newblock URL
  \url{http://www.mathnet.ru/links/5e0609d41b53e39bccea1f8711152ecd/ppi797.pdf}.

\bibitem[Krizhevsky(2009)]{Krizhevsky09learningmultiple}
Alex Krizhevsky.
\newblock Learning multiple layers of features from tiny images.
\newblock Technical report, 2009.

\bibitem[Lee et~al.(2019)Lee, Yoon, Kim, Kim, Kim, So, and Kang]{bertbio2019}
Jinhyuk Lee, Wonjin Yoon, Sungdong Kim, Donghyeon Kim, Sunkyu Kim, Chan~Ho So,
  and Jaewoo Kang.
\newblock {BioBERT: a pre-trained biomedical language representation model for
  biomedical text mining}.
\newblock \emph{Bioinformatics}, 36\penalty0 (4):\penalty0 1234--1240, 09 2019.
\newblock ISSN 1367-4803.
\newblock \doi{10.1093/bioinformatics/btz682}.
\newblock URL \url{https://doi.org/10.1093/bioinformatics/btz682}.

\bibitem[Liu et~al.(2021)Liu, Jia, Qu, and Gong]{liu2021encodermi}
Hongbin Liu, Jinyuan Jia, Wenjie Qu, and Neil~Zhenqiang Gong.
\newblock Encodermi: Membership inference against pre-trained encoders in
  contrastive learning.
\newblock In \emph{Proceedings of the 2021 ACM SIGSAC Conference on Computer
  and Communications Security}, CCS '21, New York, NY, USA, 2021. Association
  for Computing Machinery.
\newblock ISBN 9781450384544.
\newblock URL \url{https://doi.org/10.1145/3460120.3484749}.

\bibitem[Maini et~al.(2021)Maini, Yaghini, and Papernot]{maini2021dataset}
Pratyush Maini, Mohammad Yaghini, and Nicolas Papernot.
\newblock Dataset inference: Ownership resolution in machine learning.
\newblock In \emph{Proceedings of ICLR 2021: 9th International Conference on
  Learning Representationsn}, 2021.

\bibitem[McAllester and Stratos(2020)]{mcallester2020mutualinformation}
David McAllester and Karl Stratos.
\newblock Formal limitations on the measurement of mutual information.
\newblock In Silvia Chiappa and Roberto Calandra, editors, \emph{Proceedings of
  the Twenty Third International Conference on Artificial Intelligence and
  Statistics}, volume 108 of \emph{Proceedings of Machine Learning Research},
  pages 875--884. PMLR, 26--28 Aug 2020.
\newblock URL \url{https://proceedings.mlr.press/v108/mcallester20a.html}.

\bibitem[Morcos et~al.(2018)Morcos, Raghu, and Bengio]{morcos2018insights}
Ari~S. Morcos, Maithra Raghu, and Samy Bengio.
\newblock Insights on representational similarity in neural networks with
  canonical correlation, 2018.

\bibitem[Netzer et~al.(2011)Netzer, Wang, Coates, Bissacco, Wu, and Ng]{svhn}
Yuval Netzer, Tao Wang, Adam Coates, Alessandro Bissacco, Bo~Wu, and Andrew~Y.
  Ng.
\newblock Reading digits in natural images with unsupervised feature learning.
\newblock In \emph{NIPS Workshop on Deep Learning and Unsupervised Feature
  Learning 2011}, 2011.
\newblock URL
  \url{http://ufldl.stanford.edu/housenumbers/nips2011_housenumbers.pdf}.

\bibitem[Orekondy et~al.(2020)Orekondy, Schiele, and
  Fritz]{orekondy2019prediction}
Tribhuvanesh Orekondy, Bernt Schiele, and Mario Fritz.
\newblock Prediction poisoning: Towards defenses against dnn model stealing
  attacks.
\newblock In \emph{International Conference on Learning Representations}, 2020.
\newblock URL \url{https://openreview.net/forum?id=SyevYxHtDB}.

\bibitem[Rabanser et~al.(2019)Rabanser, G{\"u}nnemann, and
  Lipton]{rabanser2019failing}
Stephan Rabanser, Stephan G{\"u}nnemann, and Zachary Lipton.
\newblock Failing loudly: An empirical study of methods for detecting dataset
  shift.
\newblock \emph{Advances in Neural Information Processing Systems}, 32, 2019.

\bibitem[Radford et~al.(2021)Radford, Kim, Hallacy, Ramesh, Goh, Agarwal,
  Sastry, Askell, Mishkin, Clark, Krueger, and Sutskever.]{clip}
Alec Radford, Jong~Wook Kim, C.~Hallacy, A.~Ramesh, G.~Goh, S.~Agarwal,
  G.~Sastry, A.~Askell, P.~Mishkin, J.~Clark, G.~Krueger, and I.~Sutskever.
\newblock Learning transferable visual models from natural language
  supervision.
\newblock \emph{Arxiv}, abs/2103.00020, 2021.

\bibitem[Sariyildiz et~al.(2021)Sariyildiz, Kalantidis, Larlus, and
  Alahari]{concept_generalization_2021_ICCV}
Mert~Bulent Sariyildiz, Yannis Kalantidis, Diane Larlus, and Karteek Alahari.
\newblock Concept generalization in visual representation learning.
\newblock In \emph{Proceedings of the IEEE/CVF International Conference on
  Computer Vision (ICCV)}, pages 9629--9639, October 2021.

\bibitem[Sha et~al.(2022)Sha, He, Yu, Backes, and Zhang]{ContSteal}
Zeyang Sha, Xinlei He, Ning Yu, Michael Backes, and Yang Zhang.
\newblock Can't steal? cont-steal! contrastive stealing attacks against image
  encoders.
\newblock 2022.
\newblock URL \url{https://arxiv.org/abs/2201.07513}.

\bibitem[Shafieinejad et~al.(2021)Shafieinejad, Lukas, Wang, Li, and
  Kerschbaum]{shafieinejad2021robustness}
Masoumeh Shafieinejad, Nils Lukas, Jiaqi Wang, Xinda Li, and Florian
  Kerschbaum.
\newblock On the robustness of backdoor-based watermarking in deep neural
  networks.
\newblock In \emph{Proceedings of the 2021 ACM Workshop on Information Hiding
  and Multimedia Security}, pages 177--188, 2021.

\bibitem[Sharir et~al.(2020)Sharir, Peleg, and Shoham]{bertTrainCost}
Or~Sharir, Barak Peleg, and Yoav Shoham.
\newblock The cost of training {NLP} models: {A} concise overview.
\newblock \emph{CoRR}, abs/2004.08900, 2020.
\newblock URL \url{https://arxiv.org/abs/2004.08900}.

\bibitem[Tramèr et~al.(2016)Tramèr, Zhang, Juels, Reiter, and
  Ristenpart.]{pred_apis}
Florian Tramèr, F.~Zhang, A.~Juels, M.~Reiter, and T.~Ristenpart.
\newblock Stealing machine learning models via prediction apis.
\newblock \emph{USENIX Security Symposium}, 2016.

\bibitem[van~den Oord et~al.(2018)van~den Oord, Li, and Vinyals.]{cpc}
A{\"a}ron van~den Oord, Y.~Li, and O.~Vinyals.
\newblock Representation learning with contrastive predictive coding.
\newblock \emph{ArXiv}, abs/1807.03748, 2018.

\bibitem[Wang et~al.(2019)Wang, Yao, Shan, Li, Viswanath, Zheng, and
  Zhao]{wang2019neural}
Bolun Wang, Yuanshun Yao, Shawn Shan, Huiying Li, Bimal Viswanath, Haitao
  Zheng, and Ben~Y Zhao.
\newblock Neural cleanse: Identifying and mitigating backdoor attacks in neural
  networks.
\newblock In \emph{2019 IEEE Symposium on Security and Privacy (SP)}, pages
  707--723. IEEE, 2019.

\bibitem[Wang et~al.(2022)Wang, Xu, Xu, Wang, and Zhu]{wang2022nontransferable}
Lixu Wang, Shichao Xu, Ruiqi Xu, Xiao Wang, and Qi~Zhu.
\newblock Non-transferable learning: A new approach for model ownership
  verification and applicability authorization.
\newblock In \emph{International Conference on Learning Representations}, 2022.
\newblock URL \url{https://openreview.net/forum?id=tYRrOdSnVUy}.

\end{thebibliography}

\newpage
\section*{Checklist}

\begin{enumerate}

\item For all authors...
\begin{enumerate}
  \item Do the main claims made in the abstract and introduction accurately reflect the paper's contributions and scope?
    \answerYes{}
  \item Did you describe the limitations of your work?
    \answerYes{See Appendix Section~\ref{sec:limitations}.}
  \item Did you discuss any potential negative societal impacts of your work?
    \answerYes{See Appendix Section~\ref{sec:negative-impact}.}
  \item Have you read the ethics review guidelines and ensured that your paper conforms to them?
    \answerYes{}
\end{enumerate}

\item If you are including theoretical results...
\begin{enumerate}
  \item Did you state the full set of assumptions of all theoretical results?
    \answerYes{}
	\item Did you include complete proofs of all theoretical results?
    \answerYes{}
\end{enumerate}

\item If you ran experiments...
\begin{enumerate}
  \item Did you include the code, data, and instructions needed to reproduce the main experimental results (either in the supplemental material or as a URL)?
    \answerYes{Provided with the supplementary materials.}
  \item Did you specify all the training details (e.g., data splits, hyperparameters, how they were chosen)?
    \answerYes{See Appendix~\ref{sec:additional-experiments}.}
	\item Did you report error bars (e.g., with respect to the random seed after running experiments multiple times)?
    \answerYes{}
	\item Did you include the total amount of compute and the type of resources used (e.g., type of GPUs, internal cluster, or cloud provider)?
    \answerYes{See Appendix~\ref{sec:additional-experiments}.}
\end{enumerate}

\item If you are using existing assets (e.g., code, data, models) or curating/releasing new assets...
\begin{enumerate}
  \item If your work uses existing assets, did you cite the creators?
    \answerYes{See Section \ref{sec:related-work}}
  \item Did you mention the license of the assets?
    \answerNo{All licences MIT.}
  \item Did you include any new assets either in the supplemental material or as a URL?
    \answerYes{}
  \item Did you discuss whether and how consent was obtained from people whose data you're using/curating?
    \answerNo{Used datasets are either synthetic or popular standard datasets.}
  \item Did you discuss whether the data you are using/curating contains personally identifiable information or offensive content?
    \answerNo{Used datasets are either synthetic or popular standard datasets.}
\end{enumerate}

\item If you used crowdsourcing or conducted research with human subjects...
\begin{enumerate}
  \item Did you include the full text of instructions given to participants and screenshots, if applicable?
    \answerNA{}
  \item Did you describe any potential participant risks, with links to Institutional Review Board (IRB) approvals, if applicable?
    \answerNA{}
  \item Did you include the estimated hourly wage paid to participants and the total amount spent on participant compensation?
    \answerNA{}
\end{enumerate}

\end{enumerate}

\appendix 
\newpage

\onecolumn

\title{\ourtitle (Supplement)}
\label{sec:appendix}

\section{Negative Societal Impacts}
\label{sec:negative-impact}

Our work aims to defend self-supervised models against model stealing attacks. Since we are directly defending models and aim to provide attribution, any negative societal impacts of our work are minimal. 
One potentially negative impact could be if the t-test result is inconclusive about a stolen model being stolen or if it incorrectly identifies an independent model as stolen. 
However, as shown from our results, we are consistently able to differentiate correctly between stolen and independent models and can use our metrics to further reinforce the results from dataset inference.
In terms of data and model access, we assume that the victim or a trusted third party, such as law enforcement, is responsible for running the dataset inference so that there are no privacy-related concerns.

\section{Protocol for Dataset Inference}
\label{sec:protocol}
We design the following protocol for Dataset Inference:
\begin{enumerate}
    \item Select a third-trusted party as an arbitrator for the ownership resolution.
    \item Arbitrator ensures that the train and validation sets are IID (from the same distribution) by combining the train and test sets followed by a random split into the training set for the defended model and the private validation set used for dataset inference. 
    \item Specification of the number of data points used for dataset inference: use all the data points from the validation set and the equivalent number of data points from the train set.
\end{enumerate}

\section{More Related Work}

\subsection{Membership Inference}
EncoderMI~\citep{liu2021encodermi} leverages the finding that an image encoder overfits to its pre-training dataset and returns more similar embeddings for pairs of augmented pre-training data points than for points not in the pre-training set. EncoderMI assumes some data points to be assessed as members and a shadow dataset as inputs. The first step is to create $n$ augmentations of a point from the shadow dataset and compute for it $n \choose 2$ (pair-wise) similarity scores using the embeddings extracted from a shadow encoder, which is trained on the shadow dataset. The scores form the membership feature vector for a given shadow data point. After labeling each such point as member or non-member, the membership features and the corresponding labels are used to train an inference classifier to infer if a given data point was a member or non-member of a target encoder. The early stopping is investigated as a mitigation defense against EncoderMI. The defense can reduce the effectiveness of the attack, however, at the cost of the lower performance of the defended encoder on downstream tasks.

\subsection{Non-Transferable Learning}
Non-Transferable Learning (NTL)~\citep{wang2022nontransferable} achieves ownership resolution and usage authorization by discouraging the model to generalize to data domains outside of its training data. 
To perform ownership resolution, NTL incentives the model to generalize poorly to a specific target domain, making it more likely to mis-classify on data from that domain. 
The authors argue that the model's unexpectedly poor behavior on the target domain can then be used like a watermark to claim ownership, and unlike many previous watermarking defenses, is harder to remove because the misclassification behavior is embedded in the model.
To control usage authorization, NTL intentionally degrades the model performance on the target domain.
The authors argue that this prevents users from applying the model on unauthorized data. 

\subsection{Metrics to Compare Encoders}
 
The desired metric for comparison between two representations should evaluate whether two representations are essentially similar or importantly different~\citep{ding2021grounding}.

\section{Additional Details on Experiments}
\label{sec:additional-experiments}

\subsection{Datasets Used}

\textbf{CIFAR10}~\citep{Krizhevsky09learningmultiple}: The CIFAR10 dataset consists of 32x32 colored images with 10 classes. There are 50000 training images and 10000 test images.

\textbf{CIFAR100}~\citep{Krizhevsky09learningmultiple}: The CIFAR100 dataset consists of 32x32 coloured images with 100 classes. There are 50000 training images and 10000 test images.

\textbf{SVHN}~\citep{svhn}: The SVHN dataset contains 32x32 coloured images with 10 classes. There are roughly 73000 training images, 26000 test images and 530000 "extra" images. 

\textbf{ImageNet}\citep{deng2009imagenet}: Larger sized coloured images with 1000 classes. As is commonly done, we resize all images to be of size 224x224. There are approximately 1 million training images and 50000 test images. 

\textbf{STL10}~\citep{coates2011stl10}: The STL10 dataset contains 96x96 coloured images with 10 classes. There are 5000 training images, 8000 test images, and 100000 unlabeled images. 

\subsection{Encoder similarity scores and p-values}

We present additional results for the encoder similarity scores and p-values from dataset inference vs the number of queries in Table~\ref{tab:num-queries-all}.

\subsection{Details on Experimental Setup}\label{sec:detailsexp}

We show a summary of the encoders used in our experiments in Table~\ref{tab:model-details}. 

\begin{table}[]
\caption{\textbf{Summary of our encoders.} We show all possible combinations of victim, stolen, and independent encoders along with their architectures. *, ** denotes that the models listed are equivalent. 
}
\label{tab:model-details}
\begin{center}
\begin{scriptsize}
\begin{sc}
\begin{tabular}{ccccccc}
\toprule
\multicolumn{2}{c}{\textbf{Victim Encoder}}  & \multicolumn{3}{c}{\textbf{Stolen Encoder}} & \multicolumn{2}{c}{\textbf{Independent Encoder}} \\
$D$ & Architecture & $D$ & Architecture &  \# Queries &$D$ & Architecture \\
\hline
CIFAR10 & ResNet34 & CIFAR10 & ResNet34 & 500 - 50K & CIFAR100* & ResNet18 \\
CIFAR10 & ResNet34 & SVHN & ResNet34 & 500 - 50K & CIFAR100* & ResNet18 \\
CIFAR10 & ResNet34 & STL10 & ResNet18 & 500 - 50K & SVHN & ResNet34 \\
 \cdashlinelr{1-7}
SVHN & ResNet34 & CIFAR10 & ResNet34 & 500 - 50K & CIFAR100* & ResNet18 \\
SVHN & ResNet34 & SVHN & ResNet34 & 500 - 50K & CIFAR100* & ResNet18 \\
SVHN & ResNet34 & STL10 & ResNet34 & 500 - 50K & CIFAR10 & ResNet34 \\
 \cdashlinelr{1-7}
ImageNet & ResNet50 & CIFAR10 & ResNet50 & 5K - 50K & CIFAR100** & ResNet50 \\
ImageNet & ResNet50 & SVHN & ResNet50 & 5K - 250K & CIFAR100** & ResNet50 \\
ImageNet & ResNet50 & ImageNet & ResNet50 & 5K - 250K & SVHN & ResNet50 \\
\toprule
\end{tabular}
\end{sc}
\end{scriptsize}
\end{center}
\vspace{-0.0cm}
\end{table}

The ResNet18/ResNet34 architectures used for the CIFAR10 and SVHN victim encoders and the related stolen and independent encoders used a 3x3 Conv layer of stride 1 instead of the default 7x7 Conv layer and did not use a max pooling layer. When stealing from the ImageNet victim encoder, the images used for queries were resized to be of size 224x224. Similarly when training independent ResNet50 encoders to be used with the ImageNet victim encoder, the images in the respective datasets were resized to a size of 224x224. 
 
For the results in Tables~\ref{tab:data-inference-gmm} and~\ref{tab:similarity_scores}, encoders with the highest numbers of queries for each case were used. In other words, for encoders stolen from the CIFAR10 or SVHN victim encoders, 50K queries were used while for encoders stolen from the ImageNet victim encoder, 250K queries were used. Note that since CIFAR10 does not have 250K different examples, 60K queries from the aggregated training and test set were used when stealing with the CIFAR10 dataset. Encoders with a smaller number of queries were also stolen with the numbers of queries ranging from 500 - 50K for the CIFAR10 and SVHN victim encoders, and queries ranging from 5K - 250K for the ImageNet victim encoder.

We train the SVHN and CIFAR10 victim models for 200 epochs. To train the independent and stolen encoders, we used 100 epochs. When stealing from the ImageNet victim encoder, the SGD/LARS optimizer was used while for other models, the Adam optimizer was used. The initial learning rate was kept constant in all cases and was adjusted with the Cosine Annealing scheduler. A batch size of 256 or 512 was used for training the models. The temperature parameters used varied between 0.1, 0.15, 0.2, and 0.25 with a larger temperature used for models with a higher number of queries. For all queries under 50K, the temperature was set to be 0.1. 

\begin{algorithm}[t]
  \caption{Stealing an Encoder~\citep{SSLextraction}.}\label{alg:stealing}
  
  \algorithmicrequire{Querying Dataset $\mathcal{D}$, access to a victim encoder $f_v(w; \theta_v)$.}
  
  \algorithmicensure{Stolen representation model $f_s(w; \theta_a)$} 

   \begin{algorithmic}[1]
   \STATE Initialize $f_s$ with a similar architecture as $f_v$.
   \FOR{sampled queries $\{x_k\}_{k=1}^{N} \in \mathcal{D}$}
    \STATE Query victim encoder to obtain representations:\newline $y_{v} = f_v(x_k)$ %
    \STATE Generate representations from stolen encoder: \newline $y_{s} =  f_s(x_k) $  %
    \STATE Compute loss $\mathcal{L} \left\{ y_v, y_s \right\}$. %
    \STATE Update stolen encoder parameters $\theta_s \coloneqq \theta_s - \eta \nabla_{\theta_s} \mathcal{L}$.
    \ENDFOR
    
  \end{algorithmic}
\end{algorithm}

We ran all experiments on machines equipped with an Intel® Xeon® Silver 4210 processor, 128 GB of RAM, and four
NVIDIA GeForce RTX 2080 graphics cards, running Ubuntu 18.04.

\section{Entropy Estimation}

We present the entropy estimator in Algorithm~\ref{alg:entropy-estimator} and the joint entropy estimator in Algorithm~\ref{alg:joint-entropy-estimator}. 

\begin{algorithm}[t]
  \caption{Kozachenko-Leonenko Entropy Estimator for Encoders.}\label{alg:entropy-estimator}
  
  \algorithmicrequire{Dataset $\gD$, number of data points $N \ge 2$ to be sampled from $\gD$, access to an encoder $f(\cdot)$, the Euler-Mascheroni constant $\gamma \approx 0.577$, and the gamma function $\Gamma$.}
  
  \algorithmicensure{Entropy Estimation $H$.} 

   \begin{algorithmic}[1]
   \STATE Sample $N$ data points from $\gD$: $x_1, ..., x_{N}$.
   \FOR{each sampled data point $\{x_k\}_{k=1}^{N} \in \gD$}
    \STATE Sample an augmentation $t$.
    \STATE Generate view $w_k=t(x_k)$.
    \STATE Query the encoder $f$ to generate the representation: $y_{k} = f(w_{k})$.
   \ENDFOR
   \FOR{each representation $\{y_k\}_{k=1}^{N} \in \mathbf{R}^d$}
    \STATE Find nearest neighbor distance: $R_k = ||y_k - y_i||_2$, where $i\ne k$.  
    \STATE Compute transformation: $z_k = (N-1)\cdot (R_k)^d$.
   \ENDFOR
   \STATE Compute the volume of the unit ball in $\mathbf{R}^d$: $B_d = \frac{\pi^{d/2}}{\Gamma(1 + d/2)}$.
   \STATE Compute Entropy: $H=\frac{1}{N} \sum_{i=1}^{N} \log z_i + \log B_d + \gamma$
  \end{algorithmic}
\end{algorithm}

\begin{algorithm}[t]
  \caption{Kozachenko-Leonenko \textbf{Joint} Entropy Estimator for Encoders.}\label{alg:joint-entropy-estimator}
  
  \algorithmicrequire{Dataset $\gD$, number of data points $N \ge 2$ to be sampled from $\gD$, an access to an encoder $f(\cdot)$, an access to an encoder $g(\cdot)$, the Euler-Mascheroni constant $\gamma \approx 0.577$, and the gamma function $\Gamma$.}
  
  \algorithmicensure{Joint Entropy Estimation $H$.} 

   \begin{algorithmic}[1]
   \STATE Sample $N$ data points from $\gD$: $x_1, ..., x_{N}$.
   \FOR{each sampled data point $\{x_k\}_{k=1}^{N} \in \gD$}
    \STATE Sample an augmentation $t$.
    \STATE Generate view $w_k=t(x_k)$.
    \STATE Query the encoder $f$ to generate the representation: $\hat{y}_{k} = f(w_{k})$.
    \STATE Query the encoder $g$ to generate the representation: $\bar{y}_{k} = g(w_{k})$.
    \STATE Concatenate the representations: $y_k = \hat{y}_{k} \mathbin\Vert \bar{y}_{k}$.
   \ENDFOR
   \FOR{each representation $\{y_k\}_{k=1}^{N} \in \mathbf{R}^d$}
    \STATE Find nearest neighbor distance: $R_k = ||y_k - y_i||_2$, where $i\ne k$.  
    \STATE Compute transformation: $z_k = (N-1)\cdot (R_k)^{2d}$.
   \ENDFOR
   \STATE Compute the volume of the unit ball in $\mathbf{R}^d$: $B_d = \frac{\pi^{d/2}}{\Gamma(1 + d/2)}$.
   \STATE Compute Entropy: $H=\frac{1}{N} \sum_{i=1}^{N} \log z_i + \log B_{2d} + \gamma$
  \end{algorithmic}
\end{algorithm}

\section{Linear Evaluation of Encoders}

In Table~\ref{tab:StealImagenet}, we show results for the downstream accuracies of models stolen from the encoder pre-trained on the ImageNet dataset. The extraction of the representation model is possible at a fraction of the cost with a smaller number of queries (less than one-fifth) required to train the victim model. %
In general, the performance of the stolen encoders increases with the number of queries. We also perform a similar evaluation for victim encoders trained on the CIFAR10 %
and SVHN %
datasets and models stolen from these encoders in Tables~\ref{tab:StealCIFAR10} and~\ref{tab:StealSVHN}, respectively.

\begin{table*}[t]
\caption{\textbf{Linear Evaluation Accuracy} on a victim and stolen encoders. The victim encoder is pre-trained on the ImageNet dataset. %
}
\label{tab:StealImagenet}
\begin{center}
\begin{small}
\begin{sc}
\begin{tabular}{cccccccc}
\toprule
\# of Queries & Dataset & Data Type & CIFAR10 & CIFAR100 & STL10 & SVHN & F-MNIST\\
\midrule
\textit{Victim Encoder} & N/A & N/A & 90.33 & 71.45 & 94.9 & 79.39 & 91.9 \\
\cdashlinelr{1-8}
60K & CIFAR10 & Train/Test & 83.3 & 57.0 & 71.2 & 73.8 & 90.7  \\
250K & ImageNet & Train & \textbf{80.0} & \textbf{57.0} & \textbf{85.8} & 71.5 & 90.2 \\ %
5K & SVHN & Extra & 42.0 & 16.2 & 34.4  & 26.9  & 81.3  \\
10K & SVHN & Extra & 60.8  & 33.0 & 50.5 & 71.7 & 87.5 \\ 
50K & SVHN & Extra & 73.3 & 47.1 & 58.2 & 78.8 &  90.4 \\
100K & SVHN & Extra & 76.3  & 50.2 & 61.1 & 78.2 & 90.8 \\
200K & SVHN & Extra & 76.9 & 52.0 & 62.1 & 78.3 & 90.8 \\
250K & SVHN & Extra & 77.1 & 52.6 & 61.9 & \textbf{80.2} & \textbf{91.4} \\
\bottomrule
\end{tabular}
\end{sc}
\end{small}
\end{center}
\vskip -0.1in
\end{table*}

\begin{table*}[t]
\caption{\textbf{Linear Evaluation Accuracy} on a victim and stolen encoders. The victim encoder is pre-trained on the CIFAR10 dataset. 
}
\label{tab:StealCIFAR10}
\begin{center}
\begin{small}
\begin{sc}
\begin{tabular}{cccccccc}
\toprule
\# of Queries & Dataset & CIFAR10 & STL10 & SVHN \\
\midrule
\textit{Victim Encoder} & N/A & 87.4 & 73.4 & 49.5 \\
\cdashlinelr{1-8}
50K & SVHN & 61.2 & 51.7 & 54.8 \\
50K & CIFAR10 & 84.8 & 70.7 & 52.4 \\
50K & STL10 & 86.8 & 73.0 & 49.8 \\

\bottomrule
\end{tabular}
\end{sc}
\end{small}
\end{center}
\vskip -0.1in
\end{table*}

\begin{table*}[t]
\caption{\textbf{Linear evaluation accuracy, mutual information score, cosine similarity and p-values} on a victim and stolen encoders. The victim encoder is pre-trained on the SVHN dataset. We observe higher performance on downstream tasks with more queries and similarly observe higher similarity scores for encoders stolen with more queries (see Table~\ref{tab:num-queries}). 
}
\label{tab:StealSVHN}
\begin{center}
\begin{small}
\begin{sc}
\begin{tabular}{cccccccc}
\toprule
\# of Queries & Dataset &  CIFAR10 &  STL10 & SVHN & \new{$S(\cdot,f_v)$} & \new{$C(\cdot,f_v)$} & \new{p-value} \\
\midrule
\textit{Victim Encoder} & N/A & 57.5 & 50.6 & 80.5 & \new{1} & \new{1} & \new{9.69e-227} \\
\cdashlinelr{1-8}
500 & CIFAR10 & 24.5 & 23.3 & 19.3 & \new{0.0} & \new{0.24} & \new{6.89e-1} \\
5K & CIFAR10 & 53.3 & 45.2 & 58.1 & \new{0.11} & \new{0.40} & \new{3.51e-1} \\
7K & CIFAR10 & 57.9 & 50.9 & 69.2 & \new{0.14} & \new{0.46} & \new{4.72e-1} \\
8K & CIFAR10 & 59.6 & 52.0 & 72.6 & \new{0.53} & \new{0.47} & \new{9.87e-2} \\
9K & CIFAR10 &  59.9 & 50.8 & 71.5 & \new{0.57} & \new{0.49} & \new{6.23e-2} \\
10K & CIFAR10 & 59.1 & 51.3 & 72.3 & \new{0.69} & \new{0.52} & \new{5.82e-3} \\
20K & CIFAR10 & 60.6  & 51.5 & 73.4 & \new{0.92} & \new{0.58} & \new{2.31e-7} \\
30K & CIFAR10 &  60.7 & 52.1 & 74.1 & \new{0.93} & \new{0.63} & \new{2.11e-10} \\
50K & CIFAR10 & 59.8  & 51.6 & 75.1 & \new{0.94} & \new{0.69} & \new{1.19e-17} \\
\cdashlinelr{1-8}
50K & STL10 & 60.6 & 51.5 & 76.8 & \new{0.95} & \new{0.89} & \new{1.65e-11} \\
50K & SVHN & 55.6 & 48.7 & 82.2 & \new{0.96} & \new{0.91} & \new{1.05e-75} \\
\cdashlinelr{1-8}
\textit{Independent} & CIFAR10 & 87.4 & 73.4 & 49.5 & \new{0.90} & \new{0.009} & \new{3.56e-1} \\
\textit{Independent} & CIFAR100 & 73.8 & 61.7 & 67.7 & \new{0.90} & \new{0.007} & \new{2.13e-1} \\
\textit{Independent} & STL10 & 79.4 & 73.1 & 55.8 & \new{0.84} & \new{0.015} & \new{4.62e-1} \\
\bottomrule
\end{tabular}
\end{sc}
\end{small}
\end{center}
\vskip -0.1in
\end{table*}

\new{
\begin{table*}[t]
\caption{\new{\textbf{Linear evaluation accuracy, mutual information score, cosine similarity and p-values} on a victim and stolen encoders. The victim encoder is pre-trained on the CIFAR10 dataset. We observe higher performance on downstream tasks with more queries and similarly observe higher similarity scores for encoders stolen with more queries (see Table~\ref{tab:num-queries})} %
}
\label{tab:StealSVHN}
\begin{center}
\begin{small}
\begin{sc}
\new{
\begin{tabular}{cccccccc}
\toprule
\# of Queries & Dataset &  CIFAR10 &  STL10 & SVHN & $S(\cdot,f_v)$ & $C(\cdot,f_v)$ & p-value \\
\midrule
\textit{Victim Encoder} & N/A & 87.4 & 73.4 & 49.5 & 1 & 1 & 1.37e-16 \\
\cdashlinelr{1-8}
500 & SVHN & 22.3 & 20.8 & 20.1 & 0.00 & 0.07 & 5.69e-1 \\
1K & SVHN & 26.5 & 23.4 & 29.9 & 0.06 & 0.11 &  9.02e-1 \\
2K & SVHN & 40.0 & 35.2 & 46.7 & 0.12 & 0.13 & 7.83e-1 \\
3K & SVHN & 44.4 & 39.8 & 56.3 & 0.18 & 0.21 & 8.04e-1 \\
4K & SVHN & 47.4 & 42.5 & 60.8 & 0.22 & 0.19 &  6.30e-1 \\
5K & SVHN & 49.5 & 43.8 & 60.9 & 0.29 & 0.28 & 4.28e-1 \\
10K & SVHN & 55.6 & 46.6 & 58.9 & 0.38 & 0.36 & 6.42e-1 \\
20K & SVHN & 57.7 & 48.7 & 60.4 & 0.43 & 0.39 &  9.81e-2 \\
30K & SVHN & 58.0 & 49.2 & 57.0 & 0.57 & 0.45 & 6.73e-2  \\
40K & SVHN & 61.4 & 51.1 & 55.0 & 0.69 & 0.49 & 4.83e-2 \\
50K & SVHN & 61.2 & 51.7 & 54.8 & 0.76 & 0.50 & 3.97e-2 \\
\cdashlinelr{1-8}
50K & CIFAR10 & 84.8 & 70.7 & 52.4 & 0.84 & 0.95  & 8.73e-7 \\
50K & STL10 & 86.8 & 73.0 & 49.8 & 0.89 & 0.92 &  1.04e-2 \\
\cdashlinelr{1-8}
\textit{Independent} & SVHN & 57.5 & 50.6 & 80.5 & 0.12 & 0.0001 &  2.96e-1\\
\textit{Independent} & CIFAR100 & 73.8 & 61.7 & 67.7 & 0.63 & 0.001 & 3.67e-1 \\
\textit{Independent} & STL10 & 79.4 & 73.1 & 55.8 & 0.34 & 0.001 &  5.21e-1 \\
\bottomrule
\end{tabular}
}
\end{sc}
\end{small}
\end{center}
\vskip -0.1in
\end{table*}
}

\section{Metrics for Measuring the Quality of Stolen Encoders}

This section considers additional metrics for measuring the quality of stolen encoders. As in Section~\ref{sec:quality}, we select a random sample of $N$ inputs from the training set and find the representations of the encoders on each of the inputs. The representations for each input are then centered by subtracting the mean and normalized to be unit vectors. For each of the inputs, the $\ell_p$ norm of the difference in representations by the encoders is computed where $p=1, 2, \infty$. The final value used as the metric is then the mean of these norms over all of the inputs. Tables~\ref{tab:l1dist}, ~\ref{tab:l2dist}, and~\ref{tab:linfdist} show the results obtained for the $\ell_1, \ell_2, $ and $\ell_{\infty}$ norms respectively. In a similar way, we find the cosine similarity between representations (closely related to the $\ell_2$ norm) and present results in Table~\ref{tab:cosinesim}. All norms and the cosine similarity are able to differentiate between stolen and independent encoders, however the $\ell_1$, $\ell_2$ norms and cosine similarity have a more clear difference.

\begin{table*}[t]
\caption{\textbf{$\ell_1$} distance between normalized and centered representations from Encoders. We compute $d(\cdot,f_v)$ between the representations of a given encoder (in a row) and the Victim encoder $f_v$ where $d$ is the $\ell_1$ norm.
}
\label{tab:l1dist}
\begin{center}
\begin{small}
\begin{sc}
\begin{tabular}{cccccccc}
\toprule
& \multicolumn{2}{c}{CIFAR10} &  \multicolumn{2}{c}{SVHN} &  \multicolumn{2}{c}{ImageNet} \\ %
Encoder & Dataset & $d(\cdot,f_v)$ &  Dataset & $d(\cdot,f_v)$ &  Dataset & $d(\cdot,f_v)$ \\
    \midrule
    \multirow{4}{*}{$f_s$} & SVHN & 14.74 $\pm$ 0.22 & SVHN & 6.92 $\pm$ 0.10 & SVHN & 25.78 $\pm$ 0.26   \\
    & CIFAR10 & 5.19 $\pm$ 0.05 & CIFAR10 & 11.92 $\pm$ 0.21 & CIFAR10 & 24.89 $\pm$ 0.28 \\
    & STL10 & 6.24 $\pm$ 0.08 & STL10 & 7.22 $\pm$ 0.14 & ImageNet & 16.68 $\pm$ 0.21  \\
        \cdashlinelr{1-7}
    \multirow{3}{*}{$f_i$} 
    & SVHN & 22.20 $\pm$ 0.07 & CIFAR10 &  22.17 $\pm$ 0.06   &  SVHN & 35.45 $\pm$ 0.21 \\
    & CIFAR100 & 23.65 $\pm$ 0.07 & CIFAR100 & 23.13 $\pm$ 0.07 & CIFAR100 & 38.96 $\pm$ 0.25 \\
    \bottomrule
\end{tabular}
\end{sc}
\end{small}
\end{center}
\vskip -0.1in
\end{table*}

\begin{table*}[t]
\caption{$\ell_2$ distance between normalized and centered representations from Encoders. We compute $d(\cdot,f_v)$ between the representations of a given encoder (in a row) and the Victim encoder $f_v$ where $d$ is the $\ell_2$ norm divided by 2. 
}
\label{tab:l2dist}
\begin{center}
\begin{small}
\begin{sc}
\begin{tabular}{cccccccc}
\toprule
& \multicolumn{2}{c}{CIFAR10} &  \multicolumn{2}{c}{SVHN} &  \multicolumn{2}{c}{ImageNet} \\ %
Encoder & Dataset & $d(\cdot,f_v)$ &  Dataset & $d(\cdot,f_v)$ &  Dataset & $d(\cdot,f_v)$ \\
    \midrule
    \multirow{4}{*}{$f_s$} & SVHN & 0.49 $\pm$ 0.007 & SVHN & 0.21 $\pm$ 0.003 & SVHN & 0.55 $\pm$ 0.04   \\
    & CIFAR10 & 0.16 $\pm$ 0.02 & CIFAR10 & 0.39 $\pm$ 0.007  & CIFAR10 & 0.53 $\pm$ 0.004 \\
    & STL10 & 0.19 $\pm$ 0.003 & STL10 & 0.23 $\pm$ 0.005 & ImageNet & 0.33 $\pm$ 0.004 \\
    \cdashlinelr{1-7}
    \multirow{3}{*}{$f_i$} 
    & SVHN & 0.71 $\pm$ 0.001 & CIFAR10 & 0.70 $\pm$ 0.01 &  SVHN & 0.71 $\pm$ 0.007 \\
    & CIFAR100 & 0.71 $\pm$ 0.001 & CIFAR100 & 0.71 $\pm$ 0.01 & CIFAR100 & 0.71 $\pm$ 0.007 \\
    \bottomrule
\end{tabular}
\end{sc}
\end{small}
\end{center}
\vskip -0.1in
\end{table*}

\begin{table*}[t]
\caption{\textbf{$\ell_{\infty}$} distance between normalized and centered representations from Encoders. We compute $d(\cdot,f_v)$ between the representations of a given encoder (in a row) and the Victim encoder $f_v$ where $d$ is the $\ell_{\infty}$ norm.
}
\label{tab:linfdist}
\begin{center}
\begin{small}
\begin{sc}
\begin{tabular}{cccccccc}
\toprule
& \multicolumn{2}{c}{CIFAR10} &  \multicolumn{2}{c}{SVHN} &  \multicolumn{2}{c}{ImageNet} \\ %
Encoder & Dataset & $d(\cdot,f_v)$ &  Dataset & $d(\cdot,f_v)$ &  Dataset & $d(\cdot,f_v)$ \\
    \midrule
    \multirow{4}{*}{$f_s$} & SVHN & 0.13 $\pm$ 0.003 & SVHN & 0.10 $\pm$ 0.002 & SVHN & 0.27 $\pm$ 0.005   \\
    & CIFAR10 & 0.06 $\pm$ 0.001 & CIFAR10 & 0.19 $\pm$ 0.004  & CIFAR10 & 0.26 $\pm$ 0.005 \\
    & STL10 & 0.08 $\pm$ 0.002 & STL10 & 0.12 $\pm$ 0.003 & ImageNet & 0.15 $\pm$ 0.004 \\
    \cdashlinelr{1-7}
    \multirow{3}{*}{$f_i$} 
    & SVHN & 0.256 $\pm$ 0.003 & CIFAR10 & 0.31 $\pm$ 0.006 &  SVHN & 0.33 $\pm$ 0.004 \\
    & CIFAR100 & 0.263 $\pm$ 0.003 & CIFAR100 & 0.31 $\pm$ 0.006 & CIFAR100 & 0.27 $\pm$ 0.006 \\
    \bottomrule
\end{tabular}
\end{sc}
\end{small}
\end{center}
\vskip -0.1in
\end{table*}

\begin{table*}[t]
\caption{Cosine similarity between normalized and centered representations from Encoders. We compute $\mathrm{sim}(\cdot,f_v)$ between the representations of a given encoder (in a row) and the Victim encoder $f_v$ where $\mathrm{sim}$ is the cosine similarity.
}
\label{tab:cosinesim}
\begin{center}
\begin{small}
\begin{sc}
\begin{tabular}{cccccccc}
\toprule
& \multicolumn{2}{c}{CIFAR10} &  \multicolumn{2}{c}{SVHN} &  \multicolumn{2}{c}{ImageNet} \\ %
Encoder & Dataset & $\mathrm{sim}(\cdot,f_v)$ &  Dataset & $\mathrm{sim}(\cdot,f_v)$ &  Dataset & $\mathrm{sim}(\cdot,f_v)$ \\
    \midrule
    \multirow{4}{*}{$f_s$} & SVHN & 0.50 $\pm$ 0.01 &  SVHN & 0.91 $\pm$ 0.003 & SVHN & 0.39 $\pm$ 0.007    \\
    & CIFAR10 & 0.95 $\pm$ 0.01 & CIFAR10 & 0.69 $\pm$ 0.01  & CIFAR10 & 0.43 $\pm$ 0.008  \\
    & STL10 & 0.92 $\pm$ 0.002 & STL10 & 0.89 $\pm$ 0.005 & ImageNet & 0.78 $\pm$ 0.005  \\
    \cdashlinelr{1-7}
    \multirow{3}{*}{$f_i$}
    & SVHN & 0.00013 $\pm$ 0.004 & CIFAR10 & 0.009 $\pm$ 0.002  &  SVHN & 0.002 $\pm$ 0.002  \\
    & CIFAR100 & 0.0007 $\pm$ 0.004  & CIFAR100 & -0.007 $\pm$ 0.003 & CIFAR100 & -0.0018 $\pm$ 0.002 \\
    \bottomrule
\end{tabular}
\end{sc}
\end{small}
\end{center}
\vskip -0.1in
\end{table*}

\begin{figure}[!t]
\begin{center}
\centering
\includegraphics[width=0.7\columnwidth]{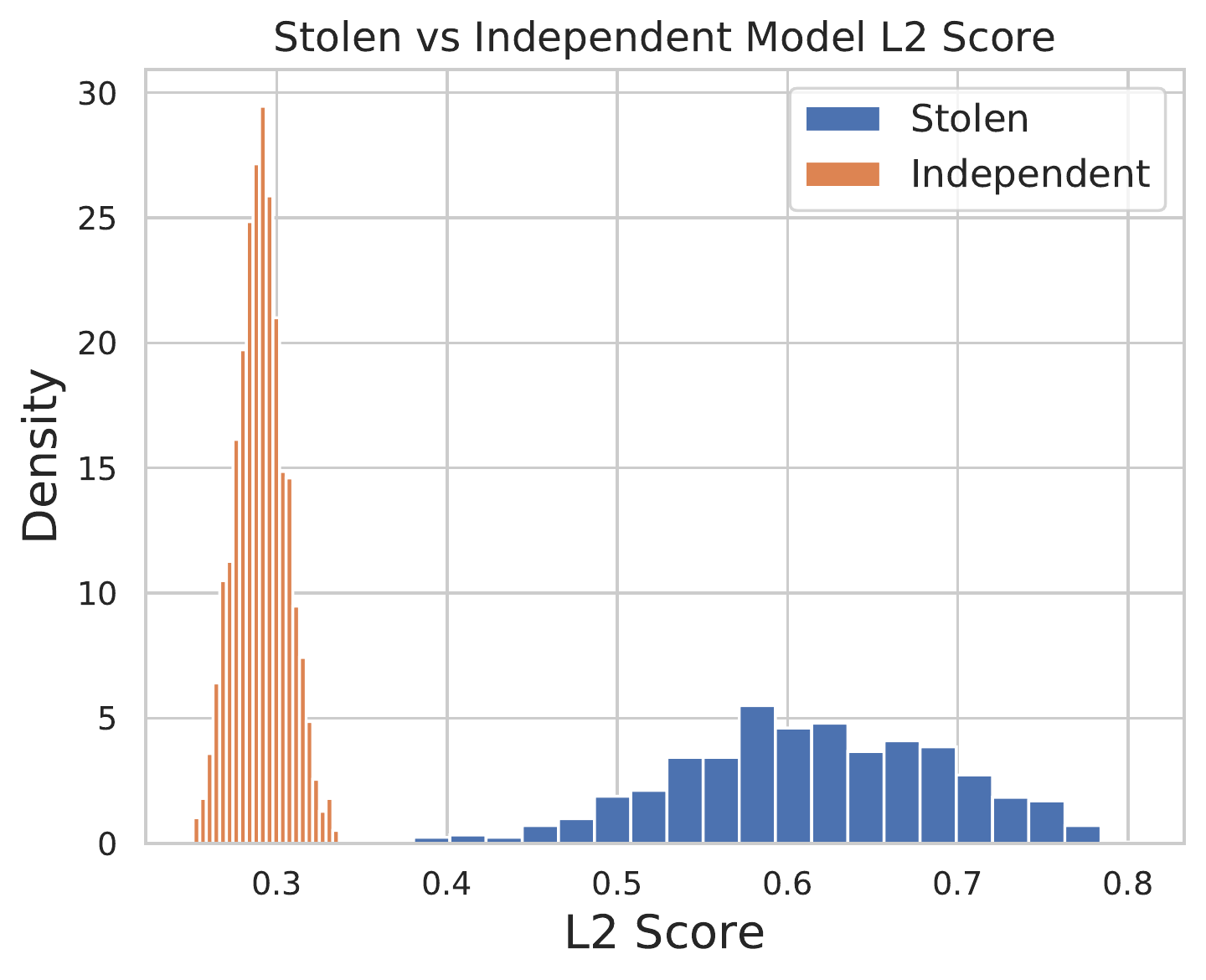}
\caption{
  Distribution of the normalized \textbf{L2 score} for the representations of an SVHN victim encoder, a stolen encoder from it (using CIFAR10 training data) and a random encoder (trained on CIFAR100). There is a pronounced difference in the distribution of the distances between stolen and independent encoders. This histogram relates to the values presented in Table~\ref{tab:l2dist}. %
}
\label{fig:l2dist}
\end{center}
\end{figure}

\begin{figure}[!t]
\begin{center}
\centering
\includegraphics[width=0.7\columnwidth]{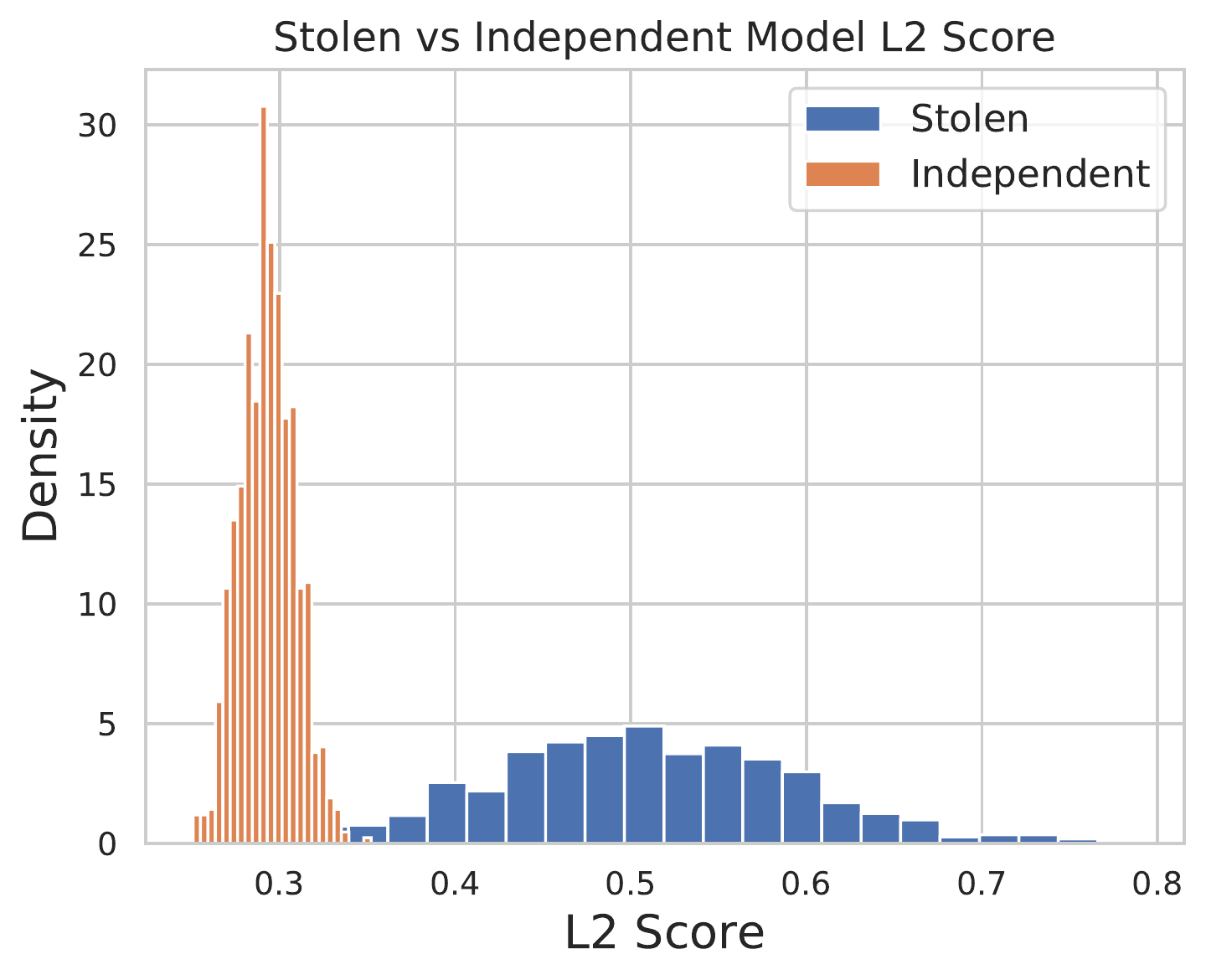}
\caption{
  Distribution of the normalized \textbf{L2 score} for the representations of a CIFAR10 victim encoder, a stolen encoder from it (using SVHN training data) and a random encoder (trained on CIFAR100). There is a pronounced difference in the distribution of the distances between stolen and independent encoders. This histogram relates to the values presented in Table~\ref{tab:l2dist}. %
}
\label{fig:l2dist}
\end{center}
\end{figure}

\begin{figure}[!t]
\begin{center}
\centering
\includegraphics[width=0.7\columnwidth]{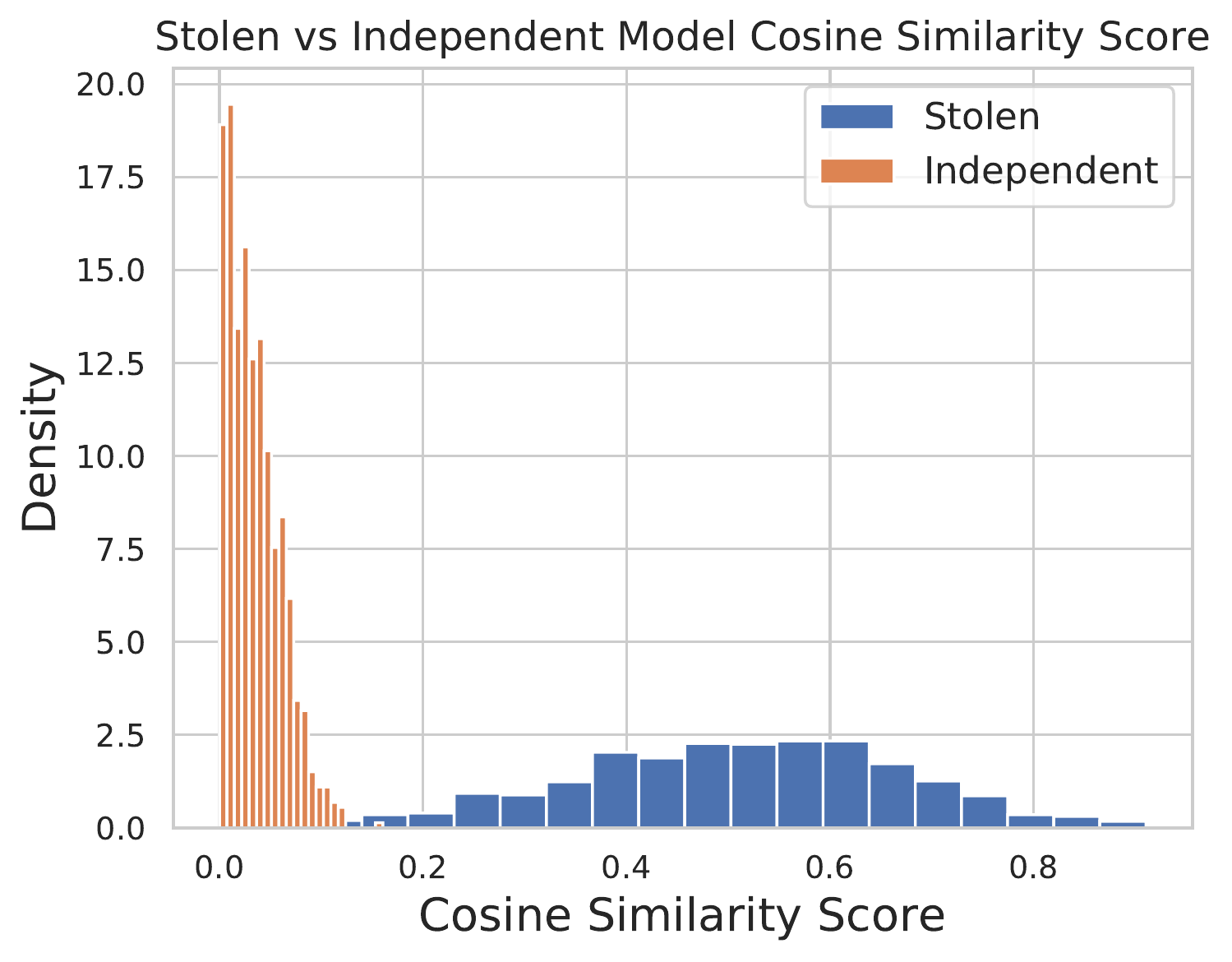}
\caption{
  Distribution of the \textbf{cosine similarity scores} for the representations of a CIFAR10 victim encoder, a stolen encoder from it (using SVHN training data) and a random encoder (trained on CIFAR100). There is a pronounced difference in the distribution of the scores between stolen and independent encoders. This histogram relates to the values presented in Table~\ref{tab:cosinesim}. %
}
\label{fig:l2dist}
\end{center}
\end{figure}

\begin{table}[t]
 \caption{\textbf{$\ell_2$ score vs the number of queries}. The quality of the stolen encoder should increase with respect to the number of queries. Therefore, we should be able to observe an increasing trend for the $\ell_2$ score.     
 }
    \label{tab:l2_score_query}
    \scriptsize
    \centering
    \begin{tabular}{cccccccc}
    \toprule
         Victim Encoder & Stolen dataset &\multicolumn{6}{c}{Number of Queries} \\
         \midrule
         \multirow{2}{*}{CIFAR10} & \multirow{2}{*}{SVHN} & 500 & 5K & 10K & 20K & 30K & 50K \\ \cdashlinelr{3-8}
         & & 0.322 $\pm$ 0.005 & 0.372 $\pm$ 0.005 & 0.411 $\pm$ 0.006 & 0.421 $\pm$ 0.006 & 0.475 $\pm$ 0.007 & 0.511 $\pm$ 0.007 \\
         \multirow{2}{*}{CIFAR10} & \multirow{2}{*}{STL10} & 500 & 5K & 10K & 20K & 30K & 50K \\ \cdashlinelr{3-8}
         & & 0.564 $\pm$ 0.007 & 0.749 $\pm$ 0.004 & 0.781 $\pm$ 0.003  & 0.795 $\pm$ 0.003 & 0.811 $\pm$ 0.002 & 0.807 $\pm$ 0.003 \\
         \midrule
          \multirow{2}{*}{SVHN} & \multirow{2}{*}{CIFAR10} & 500 & 5K & 10K & 20K & 30K & 50K \\ \cdashlinelr{3-8}
         & & 0.386 $\pm$ 0.006 & 0.459 $\pm$ 0.007  & 0.516 $\pm$ 0.007 & 0.551 $\pm$ 0.007 & 0.580 $\pm$ 0.007 & 0.614 $\pm$ 0.007 \\
         \multirow{2}{*}{SVHN} & \multirow{2}{*}{STL10} & 500 & 5K & 10K & 20K & 30K & 50K \\ \cdashlinelr{3-8}
         & & 0.458 $\pm$ 0.007 & 0.611 $\pm$ 0.007 & 0.667 $\pm$ 0.006 & 0.736 $\pm$ 0.005 & 0.738 $\pm$ 0.005 & 0.769 $\pm$ 0.005 \\
         \midrule
         \multirow{2}{*}{ImageNet} & \multirow{2}{*}{SVHN} & 5k & 10k & 50k & 100k & 200k & 250k \\ \cdashlinelr{3-8}
         & & 0.390 $\pm$ 0.001 & 0.418 $\pm$ 0.001 & 0.444 $\pm$ 0.001 & 0.446 $\pm$ 0.001 & 0.452 $\pm$ 0.001 & 0.450 $\pm$ 0.001 \\
         \multirow{2}{*}{ImageNet} & \multirow{2}{*}{ImageNet} & 5k & 10k & 50k & 100k & 200k & 250k \\ \cdashlinelr{3-8}
         & & 0.407 $\pm$ 0.001 & 0.493 $\pm$ 0.001 & 0.448 $\pm$ 0.001 & 0.515 $\pm$ 0.001 & 0.661 $\pm $ 0.001 &  0.674 $\pm$ 0.001 \\
         
  \bottomrule
    \end{tabular}
\end{table}

\begin{table}[t]
 \caption{\textbf{Cosine similarity score vs the number of queries}. The quality of the stolen encoder should increase with respect to the number of queries. Therefore, we should be able to observe a similar increasing trend for the cosine similarity score. %
 }
    \label{tab:cosine_sim_query}
    \scriptsize
    \centering
    \begin{tabular}{cccccccc}
    \toprule
         Victim Encoder & Stolen dataset &\multicolumn{6}{c}{Number of Queries} \\
         \midrule
         \multirow{2}{*}{CIFAR10} & \multirow{2}{*}{SVHN} & 500 & 5K & 10K & 20K & 30K & 50K \\ \cdashlinelr{3-8}
         & & 0.073 $\pm$ 0.012 & 0.205 $\pm$ 0.012  & 0.294 $\pm$ 0.014 & 0.318 $\pm$ 0.014 & 0.433 $\pm$ 0.015 & 0.508 $\pm$ 0.013 \\
          \multirow{2}{*}{CIFAR10} & \multirow{2}{*}{STL10} & 500 & 5K & 10K & 20K & 30K & 50K \\ \cdashlinelr{3-8}
        & & 0.606 $\pm$ 0.012 & 0.869 $\pm$ 0.004 & 0.901 $\pm$ 0.003  & 0.913 $\pm$ 0.003 & 0.927 $\pm$ 0.002 & 0.923 $\pm$ 0.002 \\
         \midrule
         \multirow{2}{*}{SVHN} & \multirow{2}{*}{CIFAR10} & 500 & 5K & 10K & 20K & 30K & 50K \\ \cdashlinelr{3-8}
         & & 0.235 $\pm$ 0.014 & 0.400 $\pm$ 0.015 & 0.518 $\pm$ 0.013 & 0.582 $\pm$ 0.013  & 0.632 $\pm$ 0.012  & 0.689 $\pm$ 0.010  \\
         \multirow{2}{*}{SVHN} & \multirow{2}{*}{STL10} & 500 & 5K & 10K & 20K & 30K & 50K \\ \cdashlinelr{3-8}
         & & 0.396 $\pm$ 0.015 & 0.682 $\pm$ 0.012 & 0.767 $\pm$ 0.009 & 0.852 $\pm$ 0.006  & 0.856 $\pm$ 0.006  & 0.887 $\pm$ 0.005  \\
         \midrule
         \multirow{2}{*}{ImageNet} & \multirow{2}{*}{SVHN} & 5k & 10k & 50k & 100k & 200k & 250k \\\cdashlinelr{3-8}
         & & 0.254 $\pm$ 0.001 & 0.320 $\pm$ 0.001 & 0.377 $\pm$ 0.002 & 0.383 $\pm$ 0.002 & 0.395 $\pm$ 0.002 & 0.391 $\pm$ 0.002 \\
         \multirow{2}{*}{ImageNet} & \multirow{2}{*}{ImageNet} & 5k & 10k & 50k & 100k & 200k & 250k \\\cdashlinelr{3-8}
         & & 0.295 $\pm$ 0.001 & 0.480 $\pm$ 0.002  & 0.386 $\pm$ 0.002 & 0.526 $\pm$ 0.002 & 0.766 $\pm$ 0.001  & 0.783 $\pm$ 0.001 \\
  \bottomrule
    \end{tabular}
\end{table}

\subsection{Analysis of Distance and Cosine Similarity Based Metrics} \label{sec:l2analysis}

We use the $\ell_2$ score based on the $\ell_2$ norm of the distance between representations, and the cosine similarity score between representations as ways to measure the quality of stolen encoders and differentiate between stolen and independent encoders. 

We create two score metrics, namely the cosine similarity score $C$: %
\begin{equation} \label{eq: cosscore}
    C = |\mathrm{sim}(a,b)| 
\end{equation}

where $\mathrm{sim}(a,b) = \frac{a^Tb}{||a||_2||b||_2}$ is the cosine similarity between representation vectors $a$ and $b$, and the $\ell_2$ distance score which transforms the $\ell_2$ norm of the difference as: 
\begin{equation} \label{eq: l2score}
  Score_{\ell_2} = 1 - \frac{1}{2} \left\lVert \frac{a}{\lVert a \rVert_2} - \frac{b}{\lVert b \rVert_2} \right\rVert_2 
\end{equation}

We first note that there are various ways by which an attacker may steal an encoder. To simplify the analysis, we consider two main cases, one where the attacker minimizes the mean squared error between its representations and the representations returned by the victim and the other where the attacker minimizes the InfoNCE contrastive loss (other contrastive loss functions are similar). 

With the MSE loss, the attacker directly minimizes the mean squared error between its representations and the representations of the victim encoder on the queries it makes. Let $x_i$ be a query made by an attacker and $f_v(x_i)  = h_{v_i}, f_s(x_i) = h_{s_i} \in \mathcal{R}^n$ be the representations of the victim and stolen encoders respectively for this query. The MSE loss between these two representations is then $\frac{1}{n} \sum_{j=1}^{n} (h_{{v_i}_j} - h_{{s_i}_j})^2 $. Comparatively, the $\ell_2$ distance ($\ell_2$ norm of the difference) between these two representations is $||h_{v_i} - h_{s_i}||_2 = \sqrt{\sum_{j=1}^n (h_{{v_i}_j} - h_{{s_i}_j})^2}$. From these two expressions, it follows directly that minimizing the MSE loss, which is equivalent to minimizing $\sum_{j=1}^{n} (h_{{v_i}_j} - h_{{s_i}_j})^2 $, also minimizes the $\ell_2$ distance. 

We now consider the case where an attacker uses a contrastive loss function such as the InfoNCE loss~\citep{cpc}, specifically as used in~\citep{simclr}. The InfoNCE loss consists of positive and negative pairs and encourages positive pairs to have similar representations and negative pairs to have dissimilar representations. When stealing from a victim encoder, the attacker uses its own representation and the representation from the victim encoder for a single query as a positive pairs while the other inputs in the batch are considered negative pairs. Given an input batch of queries $\{x_1, x_2, \dots, x_m\}$ sent by an attacker, the representations for each query from both the victim and stolen encoders are concatenated as $\{h_{s_1}, h_{s_2}, \dots, h_{s_m}, h_{v_1}, h_{v_2}, \dots, h_{v_m} \}$ where $h_{s_i} = f_s(x_i)$ and $h_{v_i} = f_v(x_i)$ are the representations by the stolen and victim encoders respectively. The positive pairs are $(h_{s_1}, h_{v_1}), \dots, (h_{s_m}, h_{v_m})$ and the loss between a positive pair of samples $(h_{s_i}, h_{v_i})$ is defined as $l(s_i, v_i) = -\log \frac{\exp{(\mathrm{sim}(h_{s_i},h_{v_i})} / \tau )}{\sum_{k=1}^{2m} \mathbb{1}_{[k \neq s_i]} \exp{(\mathrm{sim}(h_{s_i}, h_{k}) / \tau)}}$. Here $\mathrm{sim}$ is the cosine similarity function ($\mathrm{sim} (u, v) = \frac{u^{T}v}{||u||_2||v||_2}$ ), $\tau$ is the temperature parameter, and $\mathbb{1}_{[k \neq s_i]}$ is an indicator function equal to 1 iff $k \neq s_i$ and 0 otherwise. The overall loss for the batch is the sum of the losses over each positive pair i.e. $\mathcal{L} = \frac{1}{2m} \sum_{c=1}^m [ l(s_c, v_c) + l(v_c, s_c)]$. We can combine the log terms to rewrite this loss function as $\mathcal{L} = - \frac{1}{2m} \log \frac{\prod_{c=1}^{m} (\exp{(\mathrm{sim}(h_{s_c},h_{v_c})} / \tau ))^2}{\prod_{c=1}^{m} (\sum_{k=1}^{2m} \mathbb{1}_{[k \neq s_c]} \exp{(\mathrm{sim}(h_{s_c}, h_{k}) / \tau)})(\sum_{k=1}^{2m} \mathbb{1}_{[k \neq v_c]} \exp{(\mathrm{sim}(h_{v_c}, h_{k}) / \tau)})}$. Note that $\exp(r) > 0 \, \forall r$ so that the loss $\mathcal{L}$ is always positive (the denominator contains the terms in the numerator and each term is positive so the fraction is < 1 and has a negative log). 

Simplifying $\mathcal{L}$ by combining the exponents in the numerator and using $\log \frac{a}{b} = \log a - \log b$ gives:

$\mathcal{L} =  \frac{1}{2m} (\log \prod_{c=1}^{m} (\sum_{k=1}^{2m} \mathbb{1}_{[k \neq s_c]} \exp{(\mathrm{sim}(h_{s_c}, h_{k}) / \tau)})(\sum_{k=1}^{2m} \mathbb{1}_{[k \neq v_c]} \exp{(\mathrm{sim}(h_{v_c}, h_{k}) / \tau)}) - \frac{1}{2m} (\log \exp(\sum_{c=1}^{m} (2 \cdot \mathrm{sim} (h_{s_c}, h_{v_c}) / \tau)))$

$\mathcal{L} = \frac{1}{2m} (\log \prod_{c=1}^{m} (\sum_{k=1}^{2m} \mathbb{1}_{[k \neq s_c]} \exp{(\mathrm{sim}(h_{s_c}, h_{k}) / \tau)})(\sum_{k=1}^{2m} \mathbb{1}_{[k \neq v_c]} \exp{(\mathrm{sim}(h_{v_c}, h_{k}) / \tau)}) - \frac{1}{m} (\sum_{c=1}^{m} (\mathrm{sim} (h_{s_c}, h_{v_c}) / \tau)) $

We now note that the cosine similarity is such that $-1 \leq \mathrm{sim}(a,b) \leq 1$. Therefore minimizing the loss corresponds to maximizing the sum of the cosine similarities between positive pairs $\sum_{c=1}^{m} (\mathrm{sim} (h_{s_c}, h_{v_c}) / \tau))$ (as this term is subtracted and the the overall loss is positive). This then corresponds to maximizing the individual cosine similarities $\mathrm{sim} (h_{s_c}, h_{v_c})$. In other words, the similarity between the representation of the victim and stolen encoders on the query samples $x_i$ is maximized through the loss function. Note that the first term of the loss corresponds to minimizing the similarity between negative pairs, however, we do not focus on that aspect of the loss as part of this analysis.  We now consider the relationship between the $\ell_2$ norm of the difference of two unit vectors $a$ and $b$ and the cosine similarity $\mathrm{sim} (a,b)$ through the following theorem: 

\begin{theorem} \label{thm: l2cosine-appendix}
$||a-b||_2 = \sqrt{2(1 - \mathrm{sim} (a,b))}$ \text{ for unit vectors } $a,b$.
\end{theorem}

\begin{proof} \label{thm:l2cosine-proof}
$\mathrm{sim} (a,b) = \frac{a^Tb}{||a||_2||b||_2} = a^Tb$, since $a$ and $b$ are unit vectors. %

$||a-b||_2^2 = (a-b)^T(a-b) = (a^T-b^T)(a-b) = a^Ta - a^Tb - b^Ta + b^Tb = ||a||_2 - 2 a^Tb + ||b||_2 = 1 - 2 a^Tb + 1 = 2 - 2 a^Tb = 2(1 - sim(a,b))$

$\therefore ||a-b||_2 = \sqrt{2(1 - \mathrm{sim} (a,b))}$
\end{proof}

Therefore maximizing the cosine similarity $\mathrm{sim}(h_{s_i},h_{v_i})$ through the InfoNCE loss means minimizing the $\ell_2$ norm of the difference between the normalized representations $h'_{s_i}, h'_{v_i}$. Similarly, minimizing the $\ell_2$ norm through the MSE loss corresponds to maximizing the cosine similarity. It also follows from this theorem that the $\ell_2$ distance $||a-b||_2$ is such that $0 \leq ||a-b||_2 \leq 2$. We therefore divide the $\ell_2$ distance by 2 to get the distance to be between 0 and 1. Transforming the $\ell_2$ distance into the $\ell_2$ score, allows us to relate it more closely to the cosine similarity so that an increase in the cosine similarity corresponds to an increase in the $\ell_2$ score. The metrics are also made to have ranges between 0 and 1 through these scores. %

Theorem \ref{thm: l2cosine-appendix} allows us to relate our $\ell_2$ score and cosine similarity score. We have: 

\begin{align*}
    Score_{\ell_2} &= 1 - \frac{1}{2} \left\lVert \frac{a}{\lVert a \rVert_2} - \frac{b}{\lVert b \rVert_2} \right\rVert_2 \\
    &= 1 - \frac{1}{2} \sqrt{2(1-\mathrm{sim} (a,b)) } \\
\end{align*}

This can equivalently be written as: 

\begin{align*}
    \mathrm{sim} (a,b) &= 1 - 2 (1 - Score_{\ell_2})^2 \\
    C &= | 1 - 2 (1 - Score_{\ell_2})^2 |
\end{align*}

When computing the distances and similarity, centering and normalizing the representations before computing the scores is important to get useful metrics. Centering (i.e. subtracting the mean of the elements in each representation from each element) allows for the values in the representations to be distributed in a similar way about the mean so that values above the mean are positive while values below the mean are negative. Normalizing the representations scales the values in the representations from the two encoders being compared to a similar range which then makes the metrics consistent across different encoders where representations may have different ranges of values. Moreover, normalizing the representations allows for Theorem \ref{thm: l2cosine-appendix} to be applied which then gives us bounded score metrics and a relationship between the $\ell_2$ score and cosine similarity score.

\section{Additional Figures}

In this section, we present additional figures. 

\Cref{fig:dataset-inference-overview} presents the full overview of our dataset inference method for the self-supervised models. \Cref{fig:dataset-inference-only-test} presents the resolution of the encoder ownership (this is a simplified version of~\Cref{fig:dataset-inference-overview}).

\begin{figure}[t]
\vskip 0in
\begin{center}
\centerline{\includegraphics[width=1.0\columnwidth]{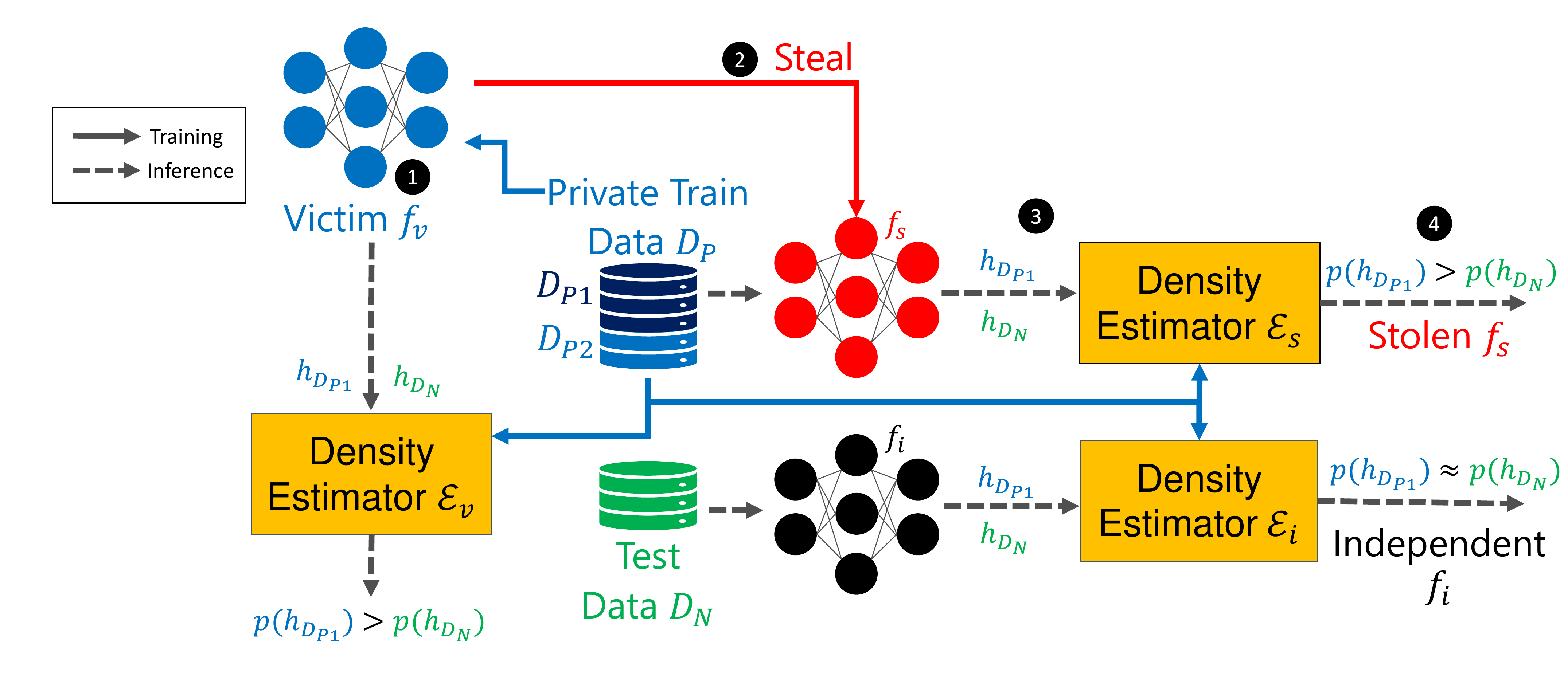}}
\caption{
\label{fig:dataset-inference-overview}
\new{
\textbf{Dataset Inference on Encoders.} 
\small{
\encircle{1} Victim trains encoder $f_v$ using private training data $D_P$.
\encircle{2} Adversary steals $f_v$: submit queries from dataset $D_S$ and obtain representations $h_{D_S}$ to train the stolen encoder $f_s$.
\encircle{3} Arbitrator trains density estimators: divide $D_P$ into non-overlapping partitions $D_{P1}$ and $D_{P2}$, and train density estimators $\gE_v$, $\gE_s$ and $\gE_i$ using the representations of $f_v$, $f_s$, $f_i$ on $D_{P2}$, respectively.
\encircle{4} Arbitrator performs dataset inference: apply $\gE_v$, $\gE_s$ and $\gE_i$ on the representations of $D_{P1}$ and $D_{N}$ of each encoder. For the victim and stolen encoders, the \prob of the representations of $D_{P1}$ is significantly higher than $D_N$, whereas, for an independent encoder, the \probs of the representations are not significantly different.
}
}
}
\end{center}
\vskip -0.4in
\end{figure}

\section{Number of Queries For Dataset Inference}

In~\Cref{tab:data-inference-fine-tune-samples}, we check how the dataset inference performs after fine-tuning the stolen model with a different number of samples.

In~\Cref{tab:data-inference-fine-tune-epochs}, we check how the dataset inference performs after fine-tuning the stolen model with a different number of epochs.

\begin{table}[t]
\caption{\new{
\textbf{
Dataset Inference for fine-tuning with a different number of samples.} 
We detect if a given encoder was stolen after fine-tuning with a different number of samples. 
We use a stolen encoder from the SVHN victim model and then retrain it with standard contrastive training using 5 epochs and data from CIFAR10.} %
}
\label{tab:data-inference-fine-tune-samples}
\scriptsize
    \centering
    \new{
    \begin{tabular}{ccc}
    \toprule
        \# of points & p-value & $\Delta \mu$ \\ %
        \midrule
        5K & 3.24e-16 & 7.03\\    %
        10K & 1.82e-16 & 7.14\\   %
        20K & 6.91e-14 & 6.09 \\   %
        50K & 5.28e-12 & 5.53\\   %
        \bottomrule
    \end{tabular}
    }
    \vskip -0.1in
\end{table}

\begin{table}[t]
\caption{\new{
\textbf{
Dataset Inference for fine-tuning with a different number of epochs.} 
We detect if a given encoder was stolen after fine-tuning with a different number of epochs. 
We use a stolen encoder from the SVHN victim model and then retrain it with standard contrastive learning using 50K data points from CIFAR10.} %
}
\label{tab:data-inference-fine-tune-epochs}
\scriptsize
    \centering
    \new{
    \begin{tabular}{ccc}
    \toprule
        \# of epochs & p-value & $\Delta \mu$ \\
        \midrule
        5 & 5.28e-12 & 5.53 \\
        10 & 8.73e-6 & 4.62 \\
        25 & 6.81e-1 & 1.34 \\
        50 & 1.73e-1  & 0.92 \\
        100 & 8.53e-1  & -0.53\\
        \bottomrule
    \end{tabular}
    }
    \vskip -0.1in
\end{table}

\end{document}